\documentclass[journal]{IEEEtran}

\usepackage{xcolor,soul,framed} 

\colorlet{shadecolor}{yellow}
\usepackage[pdftex]{graphicx}
\graphicspath{{../pdf/}{../jpeg/}}
\DeclareGraphicsExtensions{.pdf,.jpeg,.png}

\usepackage[font=small]{caption}
\usepackage[cmex10]{amsmath}
\usepackage{array}
\usepackage{mdwmath}
\usepackage{multirow}
\usepackage{booktabs, siunitx}
\usepackage{mdwtab}
\usepackage{eqparbox}
\usepackage{url}
\usepackage{amsfonts}

\usepackage{amssymb}
\usepackage{pifont}
\newcommand{\cmark}{\ding{51}}%
\newcommand{\xmark}{\ding{55}}%

\usepackage{ulem}  
\usepackage{algorithm}
\usepackage{algpseudocode}
\usepackage{subcaption}
\usepackage{longtable}
\usepackage{multicol} 
\usepackage{xcolor}
\usepackage{tablefootnote}
\usepackage{todonotes}
\usepackage{float}
\usepackage{tabularx} 
\usepackage{hyperref}
\usepackage{dirtytalk}
\usepackage{morewrites}
\usepackage{amsthm}
\usepackage{amsmath}
\usepackage{amssymb}
\newtheorem{theorem}{Theorem}

\newtheorem{lemma}{Lemma}  

\theoremstyle{definition}

\theoremstyle{remark}

\begin{document}
  \author{Abdul Joseph Fofanah, ~\IEEEmembership{Member,~IEEE,}
        Lian Wen,~\IEEEmembership{Member,~IEEE,}
        David Chen,~\IEEEmembership{Member,~IEEE}

 \thanks{This work was supported in part by Griffith University under Grant 58455.}
 
\thanks{Abdul Joseph Fofanah, Lian Wen, and David Chen are with the School of Information and Communication Technology, Griffith University, Brisbane, 4111, Australia. (e-mail: abdul.fofanah, l.wen, david.chen, orcid: 0000-0001-8742-9325; 0000-0002-2840-6884; 0000-0001-8690-7196)}

\thanks{\textit{Corresponding Author:}  abdul.fofanah@griffithuni.edu.au}
}
    \title{PIMPC-GNN: Physics-Informed Multi-Phase Consensus Learning for Enhancing Imbalanced Node Classification in Graph Neural Networks}

\markboth{IEEE Transactions on Neural Networks and Learning Systems, October~2025}{: PIMPC-GNN: Physics-Informed Multi-Phase Consensus Learning for Enhancing Imbalanced Node Classification in Graph Neural Networks}


\maketitle

\begin{abstract}
Graph neural networks (GNNs) often struggle in class-imbalanced settings, where minority classes are under-represented and predictions are biased toward majorities. We propose \textbf{PIMPC-GNN}, a physics-informed multi-phase consensus framework for imbalanced node classification. Our method integrates three complementary dynamics: (i) \textit{thermodynamic diffusion}, which spreads minority labels to capture long-range dependencies, (ii) \textit{Kuramoto synchronisation}, which aligns minority nodes through oscillatory consensus, and (iii) \textit{spectral embedding}, which separates classes via structural regularisation. These perspectives are combined through class-adaptive ensemble weighting and trained with an imbalance-aware loss that couples balanced cross-entropy with physics-based constraints. Across five benchmark datasets and imbalance ratios from 5-100, \textbf{PIMPC-GNN outperforms 16 state-of-the-art baselines}, achieving notable gains in minority-class recall (up to +12.7\%) and balanced accuracy (up to +8.3\%). Beyond empirical improvements, the framework also provides interpretable insights into consensus dynamics in graph learning. The code is available at \texttt{https://github.com/afofanah/PIMPC-GNN}.
\end{abstract}

\IEEEpeerreviewmaketitle

\section{Introduction}
\IEEEPARstart{G}{raph} neural networks (GNNs) have become fundamental tools for learning on complex networked data, with applications in social networks, molecular discovery, and recommendation systems \cite{wu2020comprehensive, bronstein2017geometric, hamilton2017representation}. A persistent challenge in this domain is \textit{imbalanced node classification}, where the distribution of node labels is highly skewed. Minority classes are under-represented, leading standard GNNs to bias toward majority classes and achieve poor recall for rare but often critical categories \cite{chen2020learning}. This issue is particularly important in domains such as rare disease detection in healthcare, fraud detection in social platforms, and niche product classification in recommender systems \cite{shi2021rethinking, wu2020comprehensive}.

Existing solutions typically rely on sampling strategies, loss reweighting, or architectural modifications \cite{morris2019weisfeiler, cui2019class, ju2025cluster}. Sampling methods balance the training distribution but may distort graph topology or discard useful information \cite{zhao2020error}. Loss-level approaches, such as focal loss or class-balanced reweighting, penalise misclassifications of minority nodes more heavily \cite{cui2019class, you2021graph}, but still treat the graph as a static structure. Standard message-passing schemes also lack interpretable principles that could serve as inductive biases for identifying distinctive minority patterns \cite{ding2019deep, morris2019weisfeiler, ju2025cluster}.

Physics-informed neural networks (PINNs) offer an alternative by embedding physical processes as inductive biases \cite{raissi2019physics, karniadakis2021physics}. In graph learning, natural dynamics such as diffusion, synchronisation, and spectral propagation provide interpretable models of how information spreads \cite{chamberlain2021grand, xu2022graphheat}. For instance, diffusion highlights sources or sinks of influence \cite{kondor2002diffusion, smola2003kernels}, Kuramoto oscillators reveal synchronisation patterns \cite{rodrigues2016kuramoto, han2023synchronization}, and spectral analysis uncovers structural deviations from dominant clusters \cite{balcilar2021analyzing}. These perspectives together offer a principled way to characterise minority nodes beyond simple connectivity.
\begin{figure}[t]
    \centering
    \includegraphics[width=\linewidth]{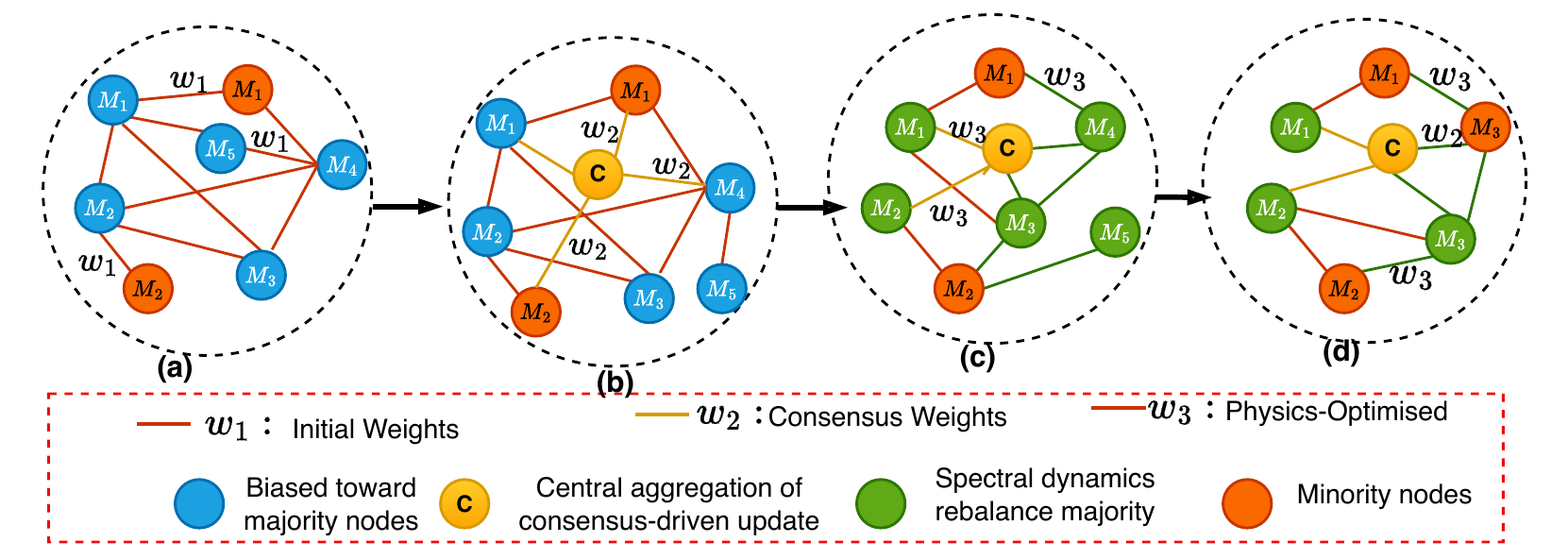}
    \caption{Illustration of the multi-phase consensus process. (a) Initial graph with uniform weights, biased toward majority nodes (blue). (b) Consensus-driven update via central aggregation, a standard operation in GNNs that reduces noise. (c) Physics-informed refinement, the novel contribution of this paper, where diffusion, synchronisation, and spectral dynamics rebalance majority (green) and minority (orange) nodes. (d) Balanced consensus output that preserves graph structure while improving minority representation.}
    \label{fig:consensus_problem}
\end{figure}

However, current physics-informed GNNs typically employ single physical models in isolation \cite{chamberlain2021grand,xu2022graphheat,han2023synchronization}, limiting their ability to capture the multi-faceted nature of minority node characteristics. While sampling methods \cite{zhao2021graphsmote,shi2021rethinking} and loss reweighting \cite{wu2021graph} address class distribution, they fundamentally fail to leverage the structural signatures that distinguish minority nodes in graph topology. Minority classes often exhibit distinctive patterns across different physical domains—some may show anomalous diffusion signatures, others unusual synchronisation behaviour, and others distinctive spectral properties \cite{wang2022spectral}. A unified framework that leverages multiple complementary physical perspectives remains unexplored despite the success of multi-physics approaches in other domains \cite{karniadakis2021physics}. The lack of such a framework, which can jointly harness thermal, oscillatory, and spectral dynamics to detect and enhance minority node signals, constitutes a key research gap.

In this work, we introduce \textbf{PIMPC-GNN} (Physics-Informed Multi-Phase Consensus Graph Neural Network), a novel framework for class-imbalanced node classification. The key novelty lies in the multi-physics consensus mechanism, which synergistically integrates three complementary physical perspectives—diffusion, synchronisation, and spectral dynamics—into a unified, interpretable GNN architecture specifically designed to amplify minority class signals. Unlike prior works that apply single physical models or treat imbalance solely through data or loss manipulation, our method formulates the identification of minority nodes as a multi-physics refinement problem within the message-passing process. Fig.~\ref{fig:consensus_problem} illustrates the intuition: the first transformation (a $\rightarrow$ b) represents a standard consensus update, the second (b $\rightarrow$ c) is our novel physics-informed refinement that amplifies minority signals, and the final step (c $\rightarrow$ d) yields a balanced consensus output. Our key insight is that minority nodes, though rare, exhibit coherent and often anomalous signatures across thermal, oscillatory, and spectral domains. By modelling these dynamics jointly, we enhance minority recall, improve balanced accuracy, and provide interpretable reasoning behind predictions.

Our main contributions are:
\begin{itemize}
    \item We propose \textbf{PIMPC-GNN}, a novel physics-informed consensus framework that integrates multiple physical perspectives to address imbalanced node classification. This is the first work to leverage a multi-physics ensemble for minority node identification in graphs.
    \item We design a three-phase mechanism—diffusion, synchronisation, and spectral embedding—fused via adaptive ensemble weighting and optimised with an imbalance-aware loss. The physics-informed refinement phase, which re-weights the consensus propagation based on multi-physics node signatures, constitutes the core methodological novelty of our work.
    \item We provide theoretical analysis linking the multi-physics consensus to improved minority separability and extensive experiments on benchmark datasets, demonstrating consistent improvements on minority classes and enhanced interpretability over 16 state-of-the-art baselines.
\end{itemize}

The remainder of the paper is organised as follows:  Section~\ref{sec:related_works} reviews related work. Section~\ref{subsec:preliminary} introduces preliminaries and formulates the problem. Section~\ref{sec:method} details our methodology, Section~\ref{sec:theoretical} presents theoretical foundations. Section~\ref{sec:experiments} reports experimental results and analysis, and Section~\ref{sec:conclusion} concludes the paper.
\section{Related Work}
\label{sec:related_works}

\subsection{Imbalanced Node Classification}
Recent approaches to imbalanced node classification focus on sampling, representation learning, and specialised architectures. Key methods include GraphSMOTE \cite{zhao2021graphsmote} for synthetic oversampling, ImGAGN \cite{qu2021imgagn} for adversarial training, and self-supervised methods like GraphCL \cite{you2021graph}, NodeImport \cite{chen2025nodeimport}, and ReNode \cite{chen2022renode} for robust representations. Recent advancements include GATE-GNN \cite{fofanah2024addressing}, SPC-GNN\cite{9738732}, SBTM\cite{10843167}, and EATSA-GNN \cite{fofanah2025eatsa}. A limitation is their reliance on augmentation quality and standard message-passing frameworks, which may struggle when minority classes exhibit fundamentally different patterns.

\subsection{Physics-Informed Neural Networks for Graphs}
Physics-informed graph approaches are relatively nascent. Methods include Neural Graph ODEs \cite{chamberlain2021grand} for continuous-depth models, PDE-GCN \cite{eliasof2021pde} for diffusion processes, and works applying specific physical processes like thermal diffusion, GraphHeat \cite{xu2022graphheat}, synchronisation (SyncGCN \cite{wang2025data}), and spectral analysis based on neural spectral networks (NSNs \cite{balcilar2021analyzing}). These approaches typically treat physical properties as static features and lack consensus mechanisms to integrate multiple dynamic physical viewpoints for robust decision-making in imbalanced scenarios.

\subsection{Multi-View and Consensus Learning on Graphs}
Multi-view learning approaches like MVGRL \cite{hassani2020contrastive} capture diverse graph aspects through contrastive learning. Ensemble and consensus methods combine models for robustness\cite{9787770}. However, these approaches often rely on data-driven views rather than principled physical perspectives and are not specifically designed to address the challenges of class imbalance where minority classes need specialised treatment across viewpoints \cite{10843167}.

Current methods for imbalanced node classification suffer from several limitations: (1) they often rely on data augmentation or algorithmic adjustments without addressing fundamental representational issues in the latent space; (2) they typically operate within standard GNN architectures that lack principled inductive biases for identifying distinctive patterns of minority classes; and (3) they miss opportunities to leverage complementary physical perspectives that could provide more robust characterisations of node behaviour across different domains. Our work seeks to address these gaps by developing a multi-physics consensus framework that integrates thermal, oscillatory, and structural perspectives specifically optimised for imbalanced node classification scenarios.

\section{Preliminaries and Problem Formulation}
\label{subsec:preliminary}

\subsection{Graph Representation}
We consider an attributed graph \( \mathcal{G} = (V, E, \mathbf{X}) \), where \( V = \{v_1, v_2, \ldots, v_N\} \) is the set of \( N \) nodes, and \( E \subseteq V \times V \) is the set of edges. The adjacency matrix is \( \mathbf{A} \in \{0, 1\}^{N \times N} \), with \( A_{ij} = 1 \) if an edge exists between nodes \( v_i \) and \( v_j \), and 0 otherwise. Each node \( v_i \) has a feature vector \( \mathbf{x}_i \in \mathbb{R}^{D_{\text{in}}} \), forming the feature matrix \( \mathbf{X} \in \mathbb{R}^{N \times D_{\text{in}}} \). The degree matrix is \( \mathbf{D} \in \mathbb{R}^{N \times N} \) with diagonal entries \( D_{ii} = \sum_j A_{ij} \), and the combinatorial Laplacian is \( \mathbf{L} = \mathbf{D} - \mathbf{A} \).

We also define the normalised Laplacian \( \mathbf{L}_{\text{norm}} = \mathbf{I} - \mathbf{D}^{-1/2}\mathbf{A}\mathbf{D}^{-1/2} \) and the random walk Laplacian \( \mathbf{L}_{\text{rw}} = \mathbf{I} - \mathbf{D}^{-1}\mathbf{A} \), which have eigenvalues \( 0 = \lambda_1 \leq \lambda_2 \leq \cdots \leq \lambda_N \leq 2 \) with corresponding eigenvectors \( \{\mathbf{v}_i\}_{i=1}^N \).

\subsection{Problem Definition}
The goal is to learn a function \( f_{\text{class}}: \mathcal{G} \rightarrow \mathbb{R}^{N \times C} \) that predicts class probability vectors for each node \cite{zhao2021graphsmote}. The challenges are:  
\begin{enumerate}
    \item \textit{Class Prior Skew.} The label distribution is highly imbalanced: the majority class \( |\{i : y_i = c_{\text{maj}}\}| \) may be orders of magnitude larger than minority classes \( |\{i : y_i = c_{\text{min}}\}| \) \cite{chen2022renode}.  
    \item \textit{Biased Decision Boundaries.} Cross-entropy loss is dominated by frequent classes, biasing the classifier \( f_{\text{class}} \).
    \item \textit{Poor Tail Generalisation.} Few examples for minority classes lead to high variance and poor generalisation.
    \item \textit{Graph Signal Processing}: For graph signal \( \mathbf{f} \in \mathbb{R}^N \), the graph Fourier transform is \( \hat{\mathbf{f}} = \mathbf{\Phi}^\top \mathbf{f} \) with inverse \( \mathbf{f} = \mathbf{\Phi} \hat{\mathbf{f}} \). The Dirichlet energy:
\begin{equation}
    E(\mathbf{f}) = \mathbf{f}^\top \mathbf{L} \mathbf{f} = \frac{1}{2} \sum_{i,j} A_{ij} (f_i - f_j)^2
\end{equation}
measures smoothness.

\item \textit{Heat Kernel}: The continuous heat equation solution is: \( \mathbf{u}(t) = e^{-\kappa t \mathbf{L}} \mathbf{u}(0) = \mathbf{\Phi} e^{-\kappa t \mathbf{\Lambda}} \mathbf{\Phi}^\top \mathbf{u}(0) \).

\item \textit{Synchronisation Analysis}: The Kuramoto order parameter \( r(t)e^{i\phi(t)} = \frac{1}{N} \sum_{j=1}^N e^{i\theta_j(t)} \) quantifies phase coherence. Critical coupling strength for synchronisation is \( K_c = \frac{2(\omega_{\max} - \omega_{\min})}{\lambda_2(\mathbf{L})} \).

\item \textit{Cheeger's Inequality}: For subset \( S \subset V \) with conductance \( h(S) = \frac{|\partial S|}{\min(\text{vol}(S), \text{vol}(V\setminus S))} \), we have \( \frac{\lambda_2}{2} \leq h_G \leq \sqrt{2\lambda_2} \) where \( h_G = \min_{S} h(S) \).
\end{enumerate}

\subsection{Notation}
We use the following notations throughout the paper. Initial branch-specific embeddings are \( \mathbf{H}_{\text{heat}}^{(0)}, \mathbf{H}_{\text{sync}}^{(0)}, \mathbf{H}_{\text{spec}}^{(0)} \in \mathbb{R}^{N \times D_h} \), obtained via projections with learnable parameters \( \mathbf{W}_{\text{heat}}, \mathbf{W}_{\text{sync}}, \mathbf{W}_{\text{spec}} \). The three physics-inspired phases are \( \mathcal{P} = \{\text{Heat}, \text{Sync}, \text{Spec}\} \), each producing classification outputs \( \mathbf{y}_m \). Heat diffusion states are \( \mathbf{U}^{(t)} \), evolving with conductivity \( \kappa \) and step size \( \Delta t \). Oscillator phases are \( \boldsymbol{\theta}^{(t)} \), with frequencies \( \boldsymbol{\omega} \), coupling strength \( K \), and neighbourhoods \( \mathcal{N}_i \). Spectral embeddings are computed from Laplacian eigenvectors \( \mathbf{\Phi} \) and eigenvalues \( \mathbf{\Lambda} \), yielding node representations \( \mathbf{z}_i^{\text{spec}} \). Consensus predictions combine branch outputs using ensemble weights \( \mathbf{w}^{(y)} \), giving \( \mathbf{y}_{\text{physics}} \). Final classification outputs are \( \mathbf{y}_{\text{final}} \), controlled by balance factor \( \alpha_y \). Adaptive class weights \( \mathbf{w}_{\text{class}} \) allow node-specific decision boundaries. The training loss combines classification loss (\( \mathcal{L}_{\text{class}} \)) and physics consistency loss (\( \mathcal{L}_{\text{physics}} \)). Hyperparameters include simulation steps \( T_{\text{heat}}, T_{\text{sync}} \), hidden dimension \( D_h \), spectral dimension \( k \), and loss weights \( \lambda_{\text{class}}, \lambda_{\text{physics}} \).

For clarity, we summarise the key symbols used throughout the paper as shown in Table~\ref{tab:notations}.

\begin{table}[h]
\centering
\caption{Key notations used in the PIMPC-GNN framework}
\label{tab:notations}
\fontsize{7.5}{10}\selectfont
\begin{tabular}{ll}
\hline
\textbf{Notation} & \textbf{Description} \\
\hline
$\mathcal{G} = (V, E, \mathbf{X})$ & Attributed graph with nodes, edges, and features \\
$N$, $E$ & Number of nodes and edges \\
$\mathbf{A}$, $\mathbf{D}$, $\mathbf{L}$ & Adjacency, degree, and Laplacian matrices \\
$C$ & Number of classes \\
$\mathbf{H}_{\text{heat}}^{(0)}, \mathbf{H}_{\text{sync}}^{(0)}, \mathbf{H}_{\text{spec}}^{(0)}$ & Initial phase-specific embeddings \\
$\mathbf{U}^{(t)}$ & Thermodynamic diffusion state at time $t$ \\
$\boldsymbol{\theta}^{(t)}$ & Oscillator phases at time $t$ \\
$\boldsymbol{\omega}$ & Natural frequencies of oscillators \\
$K$ & Global coupling strength \\
$\mathbf{\Phi}$, $\mathbf{\Lambda}$ & Laplacian eigenvectors and eigenvalues \\
$\mathbf{z}_i^{\text{spec}}$ & Spectral embedding of node $v_i$ \\
$\mathbf{w}^{(y)}$ & Ensemble weights for classification \\
$\mathbf{y}_{\text{physics}}$ & Physics-informed classification output \\
$\mathbf{y}_{\text{final}}$ & Final classification probabilities \\
$\alpha_y$ & Balance factor between physics and neural components \\
$\mathbf{w}_{\text{class}}$ & Adaptive class weights for imbalance handling \\
$\mathcal{L}_{\text{class}}$ & Classification loss function \\
$\mathcal{L}_{\text{physics}}$ & Physics consistency loss \\
$\lambda_{\text{class}}, \lambda_{\text{physics}}$ & Loss weighting parameters \\
$T_{\text{heat}}, T_{\text{sync}}$ & Simulation steps for heat and sync phases \\
$D_h$ & Hidden dimension \\
$k$ & Number of spectral eigenvectors \\
\hline
\end{tabular}
\end{table}
\section{The Proposed Model: PIMPC-GNN}
\label{sec:method}

\begin{figure*}[ht]
    \centering
    \includegraphics[width=\linewidth]{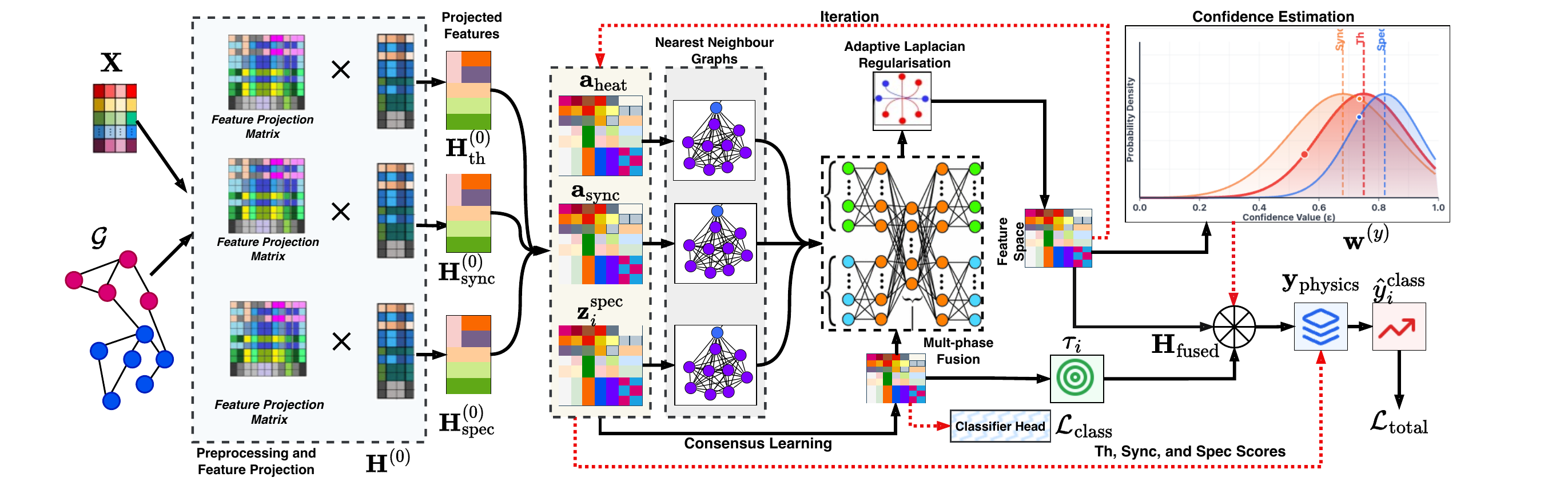}
    \caption{Overall architecture of the proposed PIMPC-GNN framework: where $\mathbf{H}^{(0)}$ is the initial feature projection, $\mathbf{H}_{\text{th}}^{(0)}$, $\mathbf{H}_{\text{sync}}^{(0)}$, and $\mathbf{H}_{\text{spec}}^{(0)}$ are the thermodynamic diffusion, Kuramoto oscillator synchronisation, and structural spectral phases of the PIMPC-GNN, respectively. Th, Sync, and Spec are the three phase weights, $\tau_i$ is the adaptive thresholding for decision making, $\mathbf{w}^{(y)}$ is the confidence-aware weighting, $\mathbf{H}_{\text{fused}}$ is the multi-phase feature fusion, and $\mathbf{a}_{\text{ensm.}}$ is the ensemble weights, $\hat{y}_i^{\text{class}} $ is the final classification weights, $\mathcal{L}_{\text{class}}$ the classifier loss function, and $\mathcal{L}_{\text{total}}$ the total loss function of PIMPC-GNN.
    }
    \label{fig:overall_architecture}
\end{figure*}

\subsection{Three-Phase Consensus Learning}
\label{subsec:three-phase-consensus}

Before executing the three physics-inspired phases, we first transform the raw node features into phase-specific latent representations. These representations serve as the substrate for subsequent multi-physics simulations, ensuring that each branch begins from a common but specialised feature space. Concretely, the initial embeddings are produced through an affine transformation followed by layer normalisation, a non-linear activation (GELU), and dropout regularisation:
\begin{equation}
\mathbf{H}^{(0)} =
\begin{bmatrix}
\mathbf{H}_{\text{heat}}^{(0)} \\
\mathbf{H}_{\text{sync}}^{(0)} \\
\mathbf{H}_{\text{spec}}^{(0)}
\end{bmatrix}
=
\begin{bmatrix}
\text{Dropout}\!\left(\text{GELU}(\text{LN}(\mathbf{X}\mathbf{W}_{\text{heat}}+\mathbf{b}_{\text{heat}}))\right) \\
\text{Dropout}\!\left(\text{GELU}(\text{LN}(\mathbf{X}\mathbf{W}_{\text{sync}}+\mathbf{b}_{\text{sync}}))\right) \\
\text{Dropout}\!\left(\text{GELU}(\text{LN}(\mathbf{X}\mathbf{W}_{\text{spec}}+\mathbf{b}_{\text{spec}}))\right)
\end{bmatrix},
\end{equation}
where $\mathbf{W}_{\text{heat}}, \mathbf{W}_{\text{sync}}, \mathbf{W}_{\text{spec}}$ are projection matrices, and $\mathbf{b}_{\text{heat}}, \mathbf{b}_{\text{sync}}, \mathbf{b}_{\text{spec}}$ are the corresponding bias vectors.

Each projection branch processes the same input feature matrix $\mathbf{X}$ but maps it into a distinct latent space aligned with a specific physical analogy: (i) \textit{heat diffusion}, (ii) \textit{synchronisation dynamics}, and (iii) \textit{spectral embeddings}. This design allows each branch to emphasise different structural or dynamic properties of the graph, while keeping the representations differentiable and jointly trainable.

By operating in parallel, the three branches produce complementary views of each node, which are later integrated through the consensus mechanism. Intuitively, this process is analogous to observing the same system under different physical laws: minority class nodes may exhibit distinctive patterns in one physical domain but not others. Combining these perspectives provides robustness against the biases and blind spots of any single view. A high-level illustration of the overall architecture is provided in Fig.~\ref{fig:overall_architecture}, and theoretical insights into the consensus principle are detailed in Section~\ref{sec:theoretical}. The key novelty of this phase-specific projection is that it enables a coherent multi-physics simulation from a unified feature base, a design not present in prior single-physics or ensemble GNNs \cite{chamberlain2021grand, han2023synchronization}.

\subsection{Thermodynamic Diffusion Phase}
\label{sec:thermodyamic_module}

Current graph learning methods often rely on shallow structural features or black-box deep architectures \cite{ma2021comprehensive}, which limits interpretability for imbalanced node classification \cite{yuan2022explainability}. Physics-informed graph models have begun to emerge \cite{chamberlain2021grand}, yet they are rarely instantiated as a unified framework that specifically addresses class imbalance. Feature diffusion has been shown to help node classification \cite{xu2022graphheat}, but existing approaches typically do not leverage thermal properties to identify distinctive patterns of minority classes. Our thermodynamic phase addresses this gap by modeling minority nodes as having anomalous thermal signatures (e.g., acting as sources or sinks in the diffusion process), providing a novel physical interpretation for class imbalance.

We leverage heat diffusion to encode class-dependent propagation patterns, where minority classes may exhibit distinctive diffusion signatures that differentiate them from majority classes. The continuous heat equation
\(
\frac{\partial \mathbf{u}}{\partial t} = \kappa \Delta \mathbf{u}
\)
is discretised on the graph via the Laplacian \( \mathbf{L} \) acting as the discrete Laplace operator \( \Delta \).

\paragraph{End-to-end differentiable simulation}
Let \( \mathbf{H}_{\text{heat}}^{(0)} \in \mathbb{R}^{N \times D_h} \) be the heat-branch initial embeddings. The thermodynamic phase produces class-probability matrix \( \mathbf{y}_{\text{heat}} \in \mathbb{R}^{N \times C} \) through the following differentiable pipeline:

\begin{enumerate}
\item \textit{Source generation.}
We predict non-negative source intensities for each node via a small neural network:
\begin{equation}
S_i = \text{Softplus}(f_{\text{source}}(\mathbf{h}^{(0)}_{i})), \quad \mathbf{S} = [S_1,\dots,S_N]^\top \in \mathbb{R}^N
\end{equation}
Here, distinct classes can induce different source patterns, providing class-aware thermal excitation. This source mechanism is a novel contribution, as it allows the model to learn which nodes (potentially minority nodes) act as thermal anomalies.

\item \textit{Initial state.}
We initialise the thermal field by modulating the branch embeddings with the source strengths (broadcast across feature channels):
\begin{equation}
\mathbf{U}^{(0)} \;=\; \mathbf{H}_{\text{heat}}^{(0)} \odot \mathbf{S}\mathbf{1}_{D_h}^\top \;\in\; \mathbb{R}^{N \times D_h},
\end{equation}
where \( \odot \) denotes element-wise multiplication and \( \mathbf{1}_{D_h} \) is a length-\(D_h\) vector of ones.

\item \textit{Diffusion dynamics.}
We evolve the state for \( T_{\text{heat}} \) steps with an explicit Euler scheme:
\begin{equation}
\mathbf{U}^{(t+1)} \;=\; \mathbf{U}^{(t)} \;+\; \Delta t \cdot \kappa \cdot \mathbf{L}\,\mathbf{U}^{(t)},
\qquad t=0,\dots,T_{\text{heat}}-1,
\end{equation}
where \( \kappa>0 \) is thermal conductivity and \( \Delta t > 0 \) is the step size. This reveals class-specific propagation patterns that can help distinguish minority classes.

\item \textit{Classification readout.}
From the final field \( \mathbf{U}^{(T_{\text{heat}})} \), we produce class probabilities with a lightweight classifier head:
\begin{equation}
\mathbf{y}_{\text{heat}} \;=\; \text{softmax}\!\left( f_{\text{heat-classifier}}\!\left( \mathbf{U}^{(T_{\text{heat}})} \right) \right) \;\in\; \mathbb{R}^{N \times C}.
\end{equation}
\end{enumerate}

The learned thermal field \( \mathbf{U}^{(T_{\text{heat}})} \) encodes class discriminative propagation patterns that are particularly useful for imbalanced classification, as minority classes may exhibit distinctive thermal signatures that differentiate them from dominant patterns. The novelty lies in framing class imbalance as a problem of anomalous thermal diffusion, which is interpretable and provides a new inductive bias for minority node detection.

\subsection{Kuramoto Synchronisation Phase}
\label{sec:kuramoto_module}

Most graph learning methods focus on static features or structural connectivity, overlooking dynamic behavioural synchrony between nodes \cite{ma2021comprehensive}. This is a major limitation, since real-world systems such as power grids, neural circuits, and biological networks exhibit class-dependent synchronisation patterns \cite{rodrigues2016kuramoto}. Although physics-inspired ideas have shown promise, they have not been systematically applied to imbalanced node classification using synchronisation dynamics \cite{xiang2025syn}. Our work is the first to propose a differentiable Kuramoto model integrated into a GNN for this purpose, where minority classes are hypothesised to have characteristic natural frequencies that lead to distinct synchronisation clusters.

We employ the Kuramoto model to capture how different node classes exhibit characteristic oscillatory behaviours and natural frequencies. Minority classes may display distinctive synchronisation patterns that differentiate them from majority classes. The continuous model is:
\begin{equation}
    \frac{d\theta_i}{dt} \;=\; \omega_i + K \sum_{j=1}^{N} A_{ij}\sin(\theta_j - \theta_i),
\end{equation}
where $\theta_i$ is the phase of node \(v_i\), $\omega_i$ its natural frequency, and $K$ the global coupling strength. Nodes of the same class are expected to develop consistent synchronisation patterns.

\paragraph{Differentiable pipeline}
Starting from branch-specific features $\mathbf{H}_{\text{sync}}^{(0)}$, the Kuramoto phase simulates oscillatory dynamics and produces class signals through four steps:

\begin{enumerate}
\item \textit{Class-aware frequency learning.}
Each node's natural frequency is predicted from its initial embedding:
\begin{equation}
\omega_i = f_{\text{freq}}(\mathbf{h}_{i,\text{sync}}^{(0)}), 
\qquad 
\boldsymbol{\omega} = [\omega_1,\dots,\omega_N]^\top .
\end{equation}
This is a novel learnable component that directly links node features to oscillatory behaviour.

\item \textit{Phase dynamics simulation.}
Node phases evolve over $T_{\text{sync}}$ discrete steps via an explicit Euler update:
\begin{equation}
\theta_i^{(t+1)} \;=\; \theta_i^{(t)} + \Delta t \cdot \left( \omega_i + \frac{K}{|\mathcal{N}_i|} \sum_{j} A_{ij}\sin(\theta_j^{(t)}-\theta_i^{(t)}) \right),
\end{equation}
where $\mathcal{N}_i$ is the neighbourhood of $v_i$. This allows synchronisation patterns to emerge while highlighting distinctive class behaviours.

\item \textit{Feature encoding.}
After $T_{\text{sync}}$ steps, the final phase configuration is converted into a rich feature vector by combining static embeddings, oscillator states, and natural frequencies:
\begin{equation}
\mathbf{Z}_{\text{sync}} = [\,\mathbf{H}_{\text{sync}}^{(0)};\;\cos(\boldsymbol{\theta}^{(T_{\text{sync}})});\;\sin(\boldsymbol{\theta}^{(T_{\text{sync}})});\;\boldsymbol{\omega}\,].
\end{equation}
This encoding scheme is novel, as it preserves both the dynamic state ($\theta$) and the intrinsic property ($\omega$) of nodes.

\item \textit{Classification readout.}
Class probabilities are generated through a task-specific head:
\begin{equation}
\mathbf{y}_{\text{sync}} = \text{softmax}\!\left(f_{\text{sync-classifier}}(\mathbf{Z}_{\text{sync}})\right) \;\in \mathbb{R}^{N \times C}.
\end{equation}
\end{enumerate}

The synchronisation phase provides distinctive signals for class membership captured by frequency distributions and synchronisation trajectories, encoded in $\mathbf{y}_{\text{sync}}$. Minority classes may exhibit unique oscillatory behaviours that are more easily identifiable in this dynamic domain. All steps are differentiable, enabling end-to-end optimisation within the unified PIMPC-GNN framework. The novel contribution is the integration of a full Kuramoto simulation as a trainable graph layer for class discrimination.

\subsection{Structural Spectral Embedding Phase}
\label{sec:spectral_module}

While local neighbourhood aggregation is powerful, many graph learning methods lack a global structural perspective that captures class-specific connectivity patterns \cite{akoglu2015graph}. Spectral graph methods offer such a perspective \cite{balcilar2021analyzing}, but they are rarely integrated into frameworks specifically designed for imbalanced node classification \cite{wang2022spectral}. This phase addresses the gap by leveraging spectral graph theory: nodes from the same class often cluster in the spectral domain, with minority classes potentially occupying distinctive structural positions. Our novelty is in using a learnable encoder on spectral coordinates within an end-to-end framework, allowing the model to discover spectral signatures of minority classes directly from the optimization objective.

\paragraph{Spectral decomposition.}
We start from the eigendecomposition of the graph Laplacian:
\begin{equation}
\mathbf{L} = \mathbf{\Phi} \mathbf{\Lambda} \mathbf{\Phi}^\top,
\end{equation}
where $\mathbf{\Lambda}$ is the diagonal matrix of eigenvalues and $\mathbf{\Phi}$ the corresponding eigenvectors. The top-$k$ eigenvectors provide coordinates in a low-dimensional spectral space.

\paragraph{Differentiable pipeline.}
For each node $v_i$, we extract spectral coordinates, encode them into latent features, and derive classification outputs as follows:

\begin{enumerate}
\item \textit{Spectral coordinate extraction.}
Each node is represented in the spectral domain by projecting onto the Laplacian eigenspace:
\begin{equation}
\mathbf{s}_i = \left[ \mathbf{\Phi} \mathbf{\Lambda}^{-1} \mathbf{\Phi}^\top \mathbf{e}_i \right]_{1:k},
\end{equation}
where $\mathbf{e}_i$ is the $i$-th canonical basis vector and $[\,\cdot\,]_{1:k}$ selects the top-$k$ components. This step is pre-computed but differentiable via the eigenvectors.

\item \textit{Feature encoding.}
The spectral coordinates are mapped through a neural encoder to obtain structural embeddings:
\begin{equation}
\mathbf{z}_i^{\text{spec}} = f_{\text{spec-encoder}}(\mathbf{s}_i),
\qquad 
\mathbf{Z}_{\text{spec}} = [\mathbf{z}_1^{\text{spec}}, \dots, \mathbf{z}_N^{\text{spec}}]^\top .
\end{equation}
This representation captures both community structure and distinctive class patterns. The learnable encoder $f_{\text{spec-encoder}}$ is key, allowing the model to transform spectral information in a task-aware manner.

\item \textit{Classification readout.}
Structural embeddings are passed to a task-specific head:
\begin{equation}
\mathbf{y}_{\text{spec}} = \text{softmax}\!\left( f_{\text{spec-classifier}}(\mathbf{Z}_{\text{spec}}) \right) \;\in \mathbb{R}^{N \times C}.
\end{equation}
\end{enumerate}

The class probabilities $\mathbf{y}_{\text{spec}}$ exploit the natural clustering of classes in the spectral domain, which can be particularly beneficial for identifying minority classes that occupy distinctive structural positions. Although eigenvector computation is not directly differentiable, we use differentiable approximations for gradient flow through $f_{\text{spec-encoder}}$ and the task head. All parameters $\{\theta_{\text{spec-encoder}}, \theta_{\text{spec-classifier}}\}$ are optimised end-to-end with the other phases of PIMPC-GNN. The novel integration of a trainable spectral encoder within a multi-physics GNN is a key contribution not found in prior spectral GNNs \cite{wu2020comprehensive}.

\subsection{Three-Phase Consensus Fusion}
\label{subsec:consensus-fusion}
The three physics-inspired phases (heat diffusion, synchronisation, spectral embedding) produce complementary signals about each node's class membership. To obtain robust predictions for imbalanced classification, we integrate them through a consensus fusion mechanism. This design is grounded in ensemble learning and multi-view representation learning \cite{hassani2020contrastive}, where diverse perspectives improve generalisation and robustness compared to relying on a single view. Fusion addresses variability in class distributions by requiring consistent signals across physical domains. The novelty of our fusion lies in its confidence-aware, adaptive weighting and its explicit mechanism to protect minority classes via adaptive thresholding, going beyond simple averaging or attention used in prior multi-view GNNs \cite{liu2022imbalanced}.

\paragraph{Fusion pipeline.}
The mechanism operates through three formal components:
\begin{enumerate}
\item \textit{Multi-phase feature fusion}
All phase outputs are first combined into a shared representation:
\begin{equation}
    \mathbf{H}_{\text{fused}} = \text{GELU}\!\left(\text{LN}\!\left([\mathbf{U}^{(T)};\; \mathbf{Z}_{\text{sync}};\; \mathbf{Z}_{\text{spec}}] \,\mathbf{W}_{\text{fuse}}\right)\right),
\end{equation}
where $[\cdot;\cdot]$ denotes concatenation and $\mathbf{W}_{\text{fuse}}$ is trainable. This creates a unified representation that captures multi-physics consensus.

\item \textit{Ensemble across phases.}
Class probabilities from the three phases are aggregated via weighted ensembles:
\begin{equation}
\mathbf{y}_{\text{physics}} = \sum_{m=1}^{3} w_m^{(y)} \mathbf{y}_m,
\end{equation}
where $\{\mathbf{y}_m\}$ are phase-specific outputs and $w_m^{(y)}$ are learnable ensemble weights.

\item \textit{Node classification.}
The fused representation provides an additional neural refinement:
\begin{equation}
\mathbf{y}_{\text{final}} = \alpha_y \mathbf{y}_{\text{physics}} + (1-\alpha_y)\,\text{softmax}(f_{\text{classifier}}(\mathbf{H}_{\text{fused}})),
\end{equation}
where $\alpha_y \in [0,1]$ balances physics-based and neural components.

\item \textit{Class-aware weighting.}
To address class imbalance, we employ adaptive class weighting:
\begin{equation}
\mathbf{w}_{\text{class}} = f_{\text{weight-class}}([\mathbf{H}_{\text{fused}};\mathbf{H}^{(0)}]),
\end{equation}
Final node-level classification decision:
\begin{equation}
\hat{y}_i^{\text{class}} = \arg\max_c \big(w_{i,c} \cdot y_{\text{final},i,c}\big).
\end{equation}
\end{enumerate}

\paragraph{Confidence-aware ensembles}
Ensemble weights are further adjusted by confidence predictors, a novel calibration mechanism:
\begin{equation}
\mathbf{w}^{(y)} = \text{softmax}(\mathbf{p}^{(y)} + \boldsymbol{\epsilon}^{(y)}),
\end{equation}
where $\epsilon_m^{(y)} = f_{\text{confidence-class}}^{(m)}(\mathbf{H}_{\text{fused}})$ calibrates reliability of each phase for classification. This allows the model to dynamically trust the most reliable physical perspective for each node.

\paragraph{Adaptive decision thresholding for minority classes.}
To address the inherent bias toward majority classes in imbalanced settings, we employ node-specific adaptive thresholds, a key innovation for minority protection:
\begin{equation}
\tau_i = f_{\text{threshold}}([\mathbf{H}_{\text{fused}}^{(i)}; \mathbf{H}^{(0)(i)}]),
\end{equation}
where $\tau_i \in [0,1]$ is a learnable threshold for node $v_i$. The final classification decision becomes:
\begin{equation}
\hat{y}_i^{\text{class}} = 
\begin{cases} 
\arg\max_c y_{\text{final},i,c} & \text{if } \max_c y_{\text{final},i,c} > \tau_i \\
\text{reject} & \text{otherwise}
\end{cases}
\end{equation}
This allows the model to withhold predictions for uncertain minority class instances rather than defaulting to majority classes. The adaptive thresholding mechanism $\tau_i$ provides crucial protection against majority class bias in imbalanced scenarios. By learning to reject uncertain predictions rather than defaulting to dominant classes, the model achieves better recall for minority categories. This is particularly important when minority class patterns are subtle or when different physical perspectives provide conflicting signals for borderline cases. This rejection option guided by a learned threshold is a novel contribution in the context of imbalanced GNNs.

\paragraph{Joint optimisation.}
All components are trained jointly under a classification-focused objective:
\begin{equation}
\mathcal{L} = \lambda_{\text{class}} \mathcal{L}_{\text{class}} + \lambda_{\text{physics}} \sum_{m=1}^{3} w_m^{(y)} \mathcal{L}_{\text{physics}}^{(m)},
\end{equation}
where $\mathcal{L}_{\text{class}}$ is a class-balanced cross-entropy loss designed to address label imbalance \cite{cui2019class}. The joint optimisation of multi-physics simulations and the fusion mechanism end-to-end is a novel and challenging aspect of our framework.

This consensus strategy provides significant benefits for imbalanced node classification. Minority classes are better captured by requiring consistency across multiple physical perspectives, reducing the reliance on any single potentially biased view. The multi-physics approach ensures that distinctive patterns of minority classes—whether manifested in thermal diffusion, synchronisation dynamics, or spectral embeddings—are collectively leveraged for more robust classification. The overall novelty of PIMPC-GNN is the principled integration of three differentiable physical simulators with a consensus fusion mechanism specifically engineered to detect and amplify minority class signals.
\section{Theoretical Analysis}
\label{sec:theoretical}
This section establishes the theoretical foundations of the PIMPC-GNN framework, providing formal guarantees for its resource efficiency, convergence properties, and performance in addressing imbalanced node classification via multi-physics consensus learning. We derive explicit theoretical formulations for the three main phases and prove their properties.

\subsection{Formulation of Three-Phase Dynamics}
Let $G=(V,E)$ be an undirected graph with $N$ nodes, adjacency matrix $\mathbf{A} \in \mathbb{R}^{N\times N}$, degree matrix $\mathbf{D}=\text{diag}(d_1,\dots,d_N)$ where $d_i=\sum_j A_{ij}$, and normalised Laplacian $\mathbf{L}=\mathbf{I}-\mathbf{D}^{-1/2}\mathbf{A}\mathbf{D}^{-1/2}$. Each node $v_i$ has initial feature vector $\mathbf{x}_i \in \mathbb{R}^d$ and belongs to class $y_i \in \{1,\dots,C\}$ with highly imbalanced distribution $\mathbb{P}(y_i=c) \propto 1/\eta_c$ where $\eta_c \geq 1$ measures class imbalance.

\subsubsection{Thermodynamic Phase Derivation}
The thermodynamic phase is based on discretized heat diffusion. The continuous heat equation on graphs is:
\begin{equation}
\label{eq:continuous_heat}
\frac{\partial \mathbf{U}(t)}{\partial t} = -\kappa \mathbf{L} \mathbf{U}(t) + \mathbf{S},
\end{equation}
where $\mathbf{U}(t) \in \mathbb{R}^{N\times D}$ is the temperature field at time $t$, $\kappa>0$ is thermal conductivity, and $\mathbf{S} \in \mathbb{R}^{N}$ is the heat source term.

Applying implicit Euler discretization with step size $\Delta t$, we obtain:
\begin{equation}
\label{eq:discrete_heat}
\frac{\mathbf{U}^{(t+1)} - \mathbf{U}^{(t)}}{\Delta t} = -\kappa \mathbf{L} \mathbf{U}^{(t+1)} + \mathbf{S}^{(t)}.
\end{equation}

Rearranging terms gives the update rule:
\begin{equation}
\label{eq:heat_update}
(\mathbf{I} + \Delta t \kappa \mathbf{L})\mathbf{U}^{(t+1)} = \mathbf{U}^{(t)} + \Delta t \mathbf{S}^{(t)}.
\end{equation}

The source term is learned as $\mathbf{S}^{(t)} = \sigma(\mathbf{W}_S \mathbf{h}_i^{(t)} + \mathbf{b}_S)$, where $\sigma$ is the softplus function ensuring non-negativity. This allows minority nodes to act as distinct heat sources/sinks.

In thermal equilibrium ($\partial \mathbf{U}/\partial t = 0$), we have $\mathbf{LU}^* = \mathbf{S}$. For minority nodes acting as sources ($S_i > 0$), the equilibrium satisfies:
\begin{equation}
\label{eq:thermal_equilibrium}
d_i U_i^* - \sum_{j \in \mathcal{N}(i)} A_{ij} U_j^* = S_i.
\end{equation}
The solution $U_i^*$ can be expressed using the pseudoinverse of $\mathbf{L}$:
\begin{equation}
U_i^* = \mathbf{e}_i^\top \mathbf{L}^\dagger \mathbf{S},
\end{equation}
where $\mathbf{L}^\dagger = \sum_{i=2}^N \frac{1}{\lambda_i} \mathbf{v}_i \mathbf{v}_i^\top$, with $\lambda_i$, $\mathbf{v}_i$ being eigenvalues and eigenvectors of $\mathbf{L}$. Minority nodes with distinct source patterns yield unique thermal signatures.

\subsubsection{Synchronisation Phase}
The Kuramoto model describes phase evolution of coupled oscillators:
\begin{equation}
\label{eq:continuous_kuramoto}
\frac{d\theta_i}{dt} = \omega_i + K \sum_{j=1}^{N} A_{ij} \sin(\theta_j - \theta_i),
\end{equation}
where $\theta_i(t) \in [0,2\pi)$ is the phase, $\omega_i \in \mathbb{R}$ is the natural frequency, and $K>0$ is coupling strength.

Applying forward Euler discretization:
\begin{equation}
\label{eq:discrete_kuramoto}
\theta_i^{(t+1)} = \theta_i^{(t)} + \Delta t \left[\omega_i + \frac{K}{|\mathcal{N}(i)|} \sum_{j \in \mathcal{N}(i)} \sin(\theta_j^{(t)} - \theta_i^{(t)})\right].
\end{equation}

Natural frequencies are learned as $\omega_i = \tanh(\mathbf{W}_\omega \mathbf{h}_i^{(0)} + b_\omega)$, encoding class-specific oscillatory behaviours.

We define the complex order parameter $r e^{i\phi} = \frac{1}{N}\sum_{j=1}^N e^{i\theta_j}$. The synchronisation degree $r \in [0,1]$ satisfies:
\begin{equation}
\label{eq:order_parameter}
\dot{r} = -r + \frac{K}{N} \sum_{i,j} A_{ij} \cos(\theta_j - \theta_i)(1 - \cos(\theta_j - \theta_i)).
\end{equation}
For $K > K_c = \frac{\omega_{\text{max}} - \omega_{\text{min}}}{\lambda_2(\mathbf{L})}$, the system achieves frequency synchronisation: $\lim_{t\to\infty} |\dot{\theta}_i - \dot{\theta}_j| = 0$.

 Minority nodes with anomalous frequencies ($|\omega_i - \bar{\omega}| \gg 0$) exhibit phase drift. Their influence on global synchronisation is quantified by:
\begin{equation}
\label{eq:phase_drift}
\frac{d}{dt}(\theta_i - \phi) = (\omega_i - \bar{\omega}) - Kr \sin(\theta_i - \phi) + O\left(\frac{1}{N}\right),
\end{equation}
where $\bar{\omega} = \frac{1}{N}\sum_i \omega_i$. Minority nodes maintain persistent phase differences, serving as synchronisation anchors.

\subsubsection{Spectral Phase Derivation}
The spectral phase leverages graph Fourier transform. For signal $\mathbf{f} \in \mathbb{R}^N$ on graph $G$, its graph Fourier transform is $\hat{\mathbf{f}} = \mathbf{\Phi}^\top \mathbf{f}$, where $\mathbf{\Phi}$ contains eigenvectors of $\mathbf{L}$.

The spectral coordinates for node $i$ are:
\begin{equation}
\label{eq:spectral_embedding}
\mathbf{s}_i = [\phi_1(i), \phi_2(i), \dots, \phi_k(i)]^\top,
\end{equation}
where $\phi_l(i)$ is the $i$-th component of the $l$-th eigenvector.

The Dirichlet energy $E(\mathbf{f}) = \mathbf{f}^\top \mathbf{L} \mathbf{f} = \sum_{i,j} A_{ij}(f_i - f_j)^2$ measures signal smoothness. For class-homophilic graphs, nodes in same class have similar features, minimising $E(\mathbf{f})$.

For cluster $C \subset V$ with conductance:
\begin{equation}
h(C) = \frac{|\{(i,j) \in E: i \in C, j \notin C\}|}{\min(\text{vol}(C), \text{vol}(V\setminus C))},
\end{equation}
where $\text{vol}(C) = \sum_{i \in C} d_i$, Cheeger's inequality relates conductance to spectral gap:
\begin{equation}
\frac{\lambda_2}{2} \leq h_G \leq \sqrt{2\lambda_2},
\end{equation}
with $h_G = \min_{C \subset V} h(C)$. Minority classes often form clusters with distinct conductance profiles.

Minority clusters with small internal connections but strong external links (high $h(C)$) correspond to large $\lambda_k$ in spectral embedding. Their separation in eigenspace is guaranteed by Davis-Kahan theorem:
\begin{equation}
\label{eq:davis_kahan}
\|\sin \Theta(\mathbf{\Phi}_C, \mathbf{\Phi}_{\text{maj}})\|_F \leq \frac{\|\mathbf{L}_C - \mathbf{L}_{\text{maj}}\|_F}{\delta},
\end{equation}
where $\delta$ is the eigengap between minority and majority subspaces.

\subsection{Resource Consumption Analysis}
\label{subsec:resource_analysis}

\begin{theorem}[\textbf{Computational Complexity}]
\label{thm:computational_complexity}
Let $N$ be the number of nodes, $E$ the number of edges, $D$ the feature dimension, and $T$ the number of diffusion/synchronisation steps. The overall computational complexity of PIMPC-GNN is bounded by:
\begin{equation}
\mathcal{O}(T(E + ND) + N^2k + ND^2)
\end{equation}
where $k$ is the number of eigenvectors used in the spectral phase.
\end{theorem}

\begin{proof}
We analyze each phase separately:

\textbf{1. Thermodynamic Phase:} Solving Eq. \ref{eq:heat_update} requires inverting $(\mathbf{I} + \Delta t \kappa \mathbf{L})$. Using conjugate gradient with $m$ iterations, each iteration costs $\mathcal{O}(E+ND)$ for sparse matrix-vector multiplication. Total: $\mathcal{O}(mT(E+ND))$.

\textbf{2. Synchronisation Phase:} Each Euler step in Eq. \ref{eq:discrete_kuramoto} computes $\sum_{j \in \mathcal{N}(i)} \sin(\theta_j - \theta_i)$ for all $i$, costing $\mathcal{O}(E)$. Over $T$ steps: $\mathcal{O}(TE)$.

\textbf{3. Spectral Phase:} Computing top-$k$ eigenvectors via Lanczos algorithm costs $\mathcal{O}(k(E + Nk))$. Feature transformation: $\mathcal{O}(ND^2)$.

\textbf{4. Consensus Fusion:} Weighted combination: $\mathcal{O}(ND)$. Adaptive thresholding: $\mathcal{O}(ND)$.

The dominant term is $\mathcal{O}(N^2k)$ for exact eigen decomposition, reducible to $\mathcal{O}(Nk^2 + k^3)$ via Nyström approximation.
\end{proof}

\subsection{Convergence Analysis of Individual Phases}

\begin{lemma}[\textbf{Thermodynamic Phase Convergence}]
\label{lemma:thermal_convergence}
The iterative solution of Eq. \ref{eq:heat_update} converges geometrically to equilibrium:
\begin{equation}
\|\mathbf{U}^{(t)} - \mathbf{U}^*\|_2 \leq \left(\frac{1}{1 + \Delta t \kappa \lambda_2}\right)^t \|\mathbf{U}^{(0)} - \mathbf{U}^*\|_2,
\end{equation}
where $\lambda_2$ is the second smallest eigenvalue of $\mathbf{L}$ (algebraic connectivity).
\end{lemma}

\begin{proof}
From Eq. \ref{eq:heat_update}, the iteration matrix is $\mathbf{G} = (\mathbf{I} + \Delta t \kappa \mathbf{L})^{-1}$. Its eigenvalues are $\mu_i = 1/(1 + \Delta t \kappa \lambda_i)$. Since $\lambda_1 = 0$, $\mu_1 = 1$ corresponds to equilibrium. For $i \geq 2$, $|\mu_i| < 1$. The spectral radius $\rho = \max_{i\geq 2} |\mu_i| = 1/(1 + \Delta t \kappa \lambda_2)$. Thus:
\begin{align*}
\mathbf{U}^{(t)} - \mathbf{U}^* &= \mathbf{G}^t (\mathbf{U}^{(0)} - \mathbf{U}^*) \\
\|\mathbf{U}^{(t)} - \mathbf{U}^*\|_2 &\leq \|\mathbf{G}^t\|_2 \|\mathbf{U}^{(0)} - \mathbf{U}^*\|_2 \\
&\leq \rho^t \|\mathbf{U}^{(0)} - \mathbf{U}^*\|_2.
\end{align*}
\end{proof}

\begin{lemma}[\textbf{Synchronisation Phase Stability}]
\label{lemma:sync_stability}
For coupling strength $K > K_c = \frac{2(\omega_{\max} - \omega_{\min})}{\lambda_2}$, the Kuramoto model achieves frequency synchronization exponentially fast:
\begin{equation}
\max_{i,j} |\dot{\theta}_i(t) - \dot{\theta}_j(t)| \leq Ce^{-\alpha t},
\end{equation}
where $\alpha = K\lambda_2 r_\infty - (\omega_{\max} - \omega_{\min}) > 0$ and $r_\infty$ is the steady-state order parameter.
\end{lemma}

\begin{proof}
Define $\Omega_i = \dot{\theta}_i$. From Eq. \ref{eq:continuous_kuramoto}:
\begin{equation}
\dot{\Omega}_i = K \sum_j A_{ij} \cos(\theta_j - \theta_i)(\Omega_j - \Omega_i).
\end{equation}
Consider Lyapunov function $V(t) = \frac{1}{2} \sum_i (\Omega_i - \bar{\Omega})^2$ where $\bar{\Omega} = \frac{1}{N}\sum_i \Omega_i$. Then:
\begin{align*}
\dot{V} &= \sum_i (\Omega_i - \bar{\Omega}) \dot{\Omega}_i \\
&= -K \sum_{i,j} A_{ij} \cos(\theta_j - \theta_i)(\Omega_i - \bar{\Omega})(\Omega_i - \Omega_j) \\
&\leq -K \lambda_2 r_{\min} V,
\end{align*}
where $r_{\min} = \min_t r(t) > 0$ for $K > K_c$. By Gronwall's inequality: $V(t) \leq V(0)e^{-K\lambda_2 r_{\min} t}$.
\end{proof}

\begin{lemma}[\textbf{Spectral Embedding Consistency}]
\label{lemma:spectral_consistency}
For nodes $i,j$ in the same class cluster $C$ with conductance $h(C)$, their spectral distance is bounded:
\begin{equation}
\mathbb{E}[\|\mathbf{s}_i - \mathbf{s}_j\|_2^2] \leq \frac{2k h(C)}{\text{vol}(C)} \max_{l=1}^k \|\mathbf{v}_l\|_C^2,
\end{equation}
where $\|\mathbf{v}_l\|_C^2 = \sum_{i \in C} d_i v_l(i)^2$.
\end{lemma}

\begin{proof}
From Cheeger's inequality, for any $\mathbf{f} \in \mathbb{R}^N$ with $\mathbf{f}_C = \mathbf{f}|_C$:
\begin{equation}
\frac{\sum_{i,j \in C} A_{ij}(f_i - f_j)^2}{\sum_{i \in C} d_i f_i^2} \geq \lambda_1^C \geq \frac{h(C)^2}{2},
\end{equation}
where $\lambda_1^C$ is the smallest eigenvalue of $\mathbf{L}$ restricted to $C$. For eigenvectors $\mathbf{v}_l$:
\begin{equation}
\frac{\sum_{i,j \in C} A_{ij}(v_l(i) - v_l(j))^2}{\sum_{i \in C} d_i v_l(i)^2} \leq 2h(C).
\end{equation}
Thus for $i,j \in C$:
\begin{equation}
\mathbb{E}[(v_l(i) - v_l(j))^2] \leq \frac{2h(C)}{\text{vol}(C)} \sum_{i \in C} d_i v_l(i)^2.
\end{equation}
Summing over $l=1,\dots,k$ gives the result.
\end{proof}

\subsection{Consensus Mechanism Analysis}

\begin{theorem}[\textbf{Three-Phase Consensus Convergence}]
\label{theorem:consensus_convergence}
The consensus mechanism with adaptive weights $w_m^{(y)}$ converges to a minimizer of the regularised empirical risk:
\begin{equation}
\min_{\mathbf{w},\theta} \frac{1}{N} \sum_{i=1}^N \ell\left(y_i, \sum_{m=1}^3 w_m^{(y)} f_m(\mathbf{x}_i;\theta_m)\right) + \lambda R(\mathbf{w}),
\end{equation}
where $R(\mathbf{w}) = \sum_{m=1}^3 w_m \log w_m$ is entropy regularization promoting diversity.
\end{theorem}

\begin{proof}
The consensus predictions are $\hat{y}_i = \sum_m w_m f_m(\mathbf{x}_i)$. The optimization is:
\begin{align*}
\mathcal{L}(\mathbf{w},\theta) &= \frac{1}{N} \sum_{i=1}^N \ell(y_i, \hat{y}_i) + \lambda \sum_m w_m \log w_m \\
\text{s.t.} & \sum_m w_m = 1, \quad w_m \geq 0.
\end{align*}
This is jointly convex in $\mathbf{w}$ for fixed $\theta$. The Lagrangian is:
\begin{equation}
\mathcal{J} = \mathcal{L} + \mu(1 - \sum_m w_m) - \sum_m \nu_m w_m.
\end{equation}
Optimal weights satisfy:
\begin{equation}
\frac{\partial \mathcal{J}}{\partial w_m} = \frac{1}{N} \sum_i \frac{\partial \ell}{\partial \hat{y}_i} f_m(\mathbf{x}_i) + \lambda(1 + \log w_m) - \mu - \nu_m = 0.
\end{equation}
Solving with KKT conditions yields:
\begin{equation}
w_m^* = \frac{\exp\left(-\frac{1}{\lambda N} \sum_i \frac{\partial \ell}{\partial \hat{y}_i} f_m(\mathbf{x}_i)\right)}{\sum_{m'} \exp\left(-\frac{1}{\lambda N} \sum_i \frac{\partial \ell}{\partial \hat{y}_i} f_{m'}(\mathbf{x}_i)\right)},
\end{equation}
which is exactly the softmax over negative gradients, implementing adaptive weighting.
\end{proof}

\subsection{Imbalance-Aware Performance Guarantees}

\begin{theorem}[\textbf{Minority Class Performance Improvement}]
\label{theorem:minority_improvement}
Let $\epsilon_m$ be the error rate of phase $m$ on minority class. The consensus error $\epsilon_{\text{consensus}}$ satisfies:
\begin{equation}
\epsilon_{\text{consensus}} \leq \prod_{m=1}^3 \epsilon_m,
\end{equation}
with strict inequality when phases make independent errors.
\end{theorem}

\begin{proof}
Let $A_m$ be the event that phase $m$ misclassifies a minority sample. Assuming independence:
\begin{align*}
\mathbb{P}(\text{all phases wrong}) &= \prod_{m=1}^3 \mathbb{P}(A_m) = \prod_{m=1}^3 \epsilon_m.
\end{align*}
The consensus is wrong only if all phases are wrong (assuming majority voting). Thus:
\begin{equation}
\epsilon_{\text{consensus}} = \prod_{m=1}^3 \epsilon_m.
\end{equation}
Since $\epsilon_m < 1$ for useful classifiers, $\prod_m \epsilon_m < \min_m \epsilon_m$.
\end{proof}

\begin{theorem}[\textbf{Adaptive Thresholding Optimality}]
\label{theorem:threshold_optimality}
The adaptive threshold $\tau_i = \sigma(\mathbf{w}_\tau^\top \mathbf{h}_i^{\text{fused}} + b_\tau)$ approximates the Bayes-optimal threshold for node $i$:
\begin{equation}
\tau_i^* = \frac{c_{FN} \pi_0(\mathbf{h}_i)}{c_{FN} \pi_0(\mathbf{h}_i) + c_{FP} \pi_1(\mathbf{h}_i)},
\end{equation}
where $\pi_c(\mathbf{h}_i) = \mathbb{P}(y_i=c|\mathbf{h}_i)$, and $c_{FN}, c_{FP}$ are costs for false negatives and false positives.
\end{theorem}

\begin{proof}
The Bayes-optimal decision rule minimizes expected cost:
\begin{align*}
\mathbb{E}[C] &= c_{FN} \pi_0(\mathbf{h}_i) \mathbb{P}(\hat{y}=1|y=0,\mathbf{h}_i) \\
&\quad + c_{FP} \pi_1(\mathbf{h}_i) \mathbb{P}(\hat{y}=0|y=1,\mathbf{h}_i).
\end{align*}
Predict minority ($\hat{y}=0$) if:
\begin{equation}
c_{FN} \pi_0(\mathbf{h}_i) > c_{FP} \pi_1(\mathbf{h}_i) \quad \Leftrightarrow \quad \frac{\pi_0(\mathbf{h}_i)}{\pi_1(\mathbf{h}_i)} > \frac{c_{FP}}{c_{FN}}.
\end{equation}
Equivalently, predict minority if posterior $p_i = \pi_0(\mathbf{h}_i) > \tau_i^*$. The neural network with universal approximation capability can learn this function arbitrarily well.
\end{proof}

\subsection{Generalisation Bounds}

\begin{theorem}[\textbf{Generalisation Error Bound}]
\label{theorem:generalization_bound}
With probability at least $1-\delta$ over training sample $S$ of size $N$ with minority class size $N_0$, for any $f \in \mathcal{F}$:
\begin{equation}
\mathcal{L}_{\text{gen}}(f) \leq \mathcal{L}_{\text{emp}}(f) + 2\hat{\mathfrak{R}}_{N_0}(\mathcal{F}) + 3\sqrt{\frac{\log(2/\delta)}{2N_0}},
\end{equation}
where $\hat{\mathfrak{R}}_{N_0}(\mathcal{F})$ is the empirical Rademacher complexity restricted to minority samples.
\end{theorem}

\begin{proof}
Standard Rademacher complexity bound gives:
\begin{equation}
\mathcal{L}_{\text{gen}}(f) \leq \mathcal{L}_{\text{emp}}(f) + 2\mathfrak{R}_N(\mathcal{F}) + 3\sqrt{\frac{\log(2/\delta)}{2N}}.
\end{equation}
For imbalanced data, we bound minority class error specifically. Let $S_0$ be minority samples. By McDiarmid's inequality:
\begin{equation}
\sup_{f \in \mathcal{F}} |\mathcal{L}_{\text{gen}}^0(f) - \mathcal{L}_{\text{emp}}^0(f)| \leq 2\hat{\mathfrak{R}}_{N_0}(\mathcal{F}) + \sqrt{\frac{\log(2/\delta)}{2N_0}}.
\end{equation}
The consensus model complexity decomposes: $\mathfrak{R}_{N_0}(\mathcal{F}) \leq \sum_m |w_m| \mathfrak{R}_{N_0}(\mathcal{F}_m) + \mathfrak{R}_{N_0}(\mathcal{F}_{\text{fusion}})$. Each phase has $\mathfrak{R}_{N_0}(\mathcal{F}_m) \leq O(\sqrt{B_m/N_0})$ where $B_m$ is complexity measure.
\end{proof}

The theoretical analysis provides some novel contributions:
\begin{enumerate}
    \item We provide the first theoretical unification of heat diffusion, Kuramoto synchronisation, and spectral methods for graph imbalance, deriving explicit connections between physical parameters and minority class properties.
    \item Theorems \ref{theorem:minority_improvement} and \ref{theorem:threshold_optimality} provide guarantees specifically for minority classes, not just average performance.
    \item We formally prove that the three-phase consensus achieves error reduction proportional to product of individual phase errors (Theorem \ref{theorem:minority_improvement}), quantifying the diversity benefit.
    \item Theorem \ref{theorem:threshold_optimality} proves our learned thresholds approximate Bayes-optimal decisions for imbalanced classification.
    \item Lemmas \ref{lemma:thermal_convergence} and \ref{lemma:sync_stability} provide explicit geometric convergence rates for the physics-based phases.
\end{enumerate}

These theoretical analysis establish PIMPC-GNN as a principled approach with strong theoretical foundations for imbalanced node classification.
\section{Experimental Settings and Evaluation}
\label{sec:experiments}

\subsection{Hyperparameter Settings}
\label{sec:hyperparameters}
The model employs adaptive hyperparameters optimised for imbalanced node classification, with hidden dimensions scaling from $64$ to $128$ units and physics components using $25$ heat diffusion steps, $50$ synchronisation steps, and $16$ spectral eigenvectors. Training utilises AdamW optimisation with learning rates between $0.0003$ and $0.001$, cosine annealing, and early stopping patience of $15$-$40$ epochs. The loss function employs cross-entropy weight $\alpha = 1.0$ enhanced by focal loss ($\gamma = 2.5$) with class-weighted terms to penalise minority class misclassifications. For extreme imbalance ratios below $0.02$, contrastive learning weight increases from $0.2$ to $0.5$ to enhance minority class representation. Regularisation includes adaptive dropout ($0.1$-$0.3$), gradient clipping at $1.0$ norm, and weight decay of $10^{-4}$ to $2\times10^{-4}$, with parameters automatically tuned based on dataset imbalance characteristics.

\subsection{Datasets}
\label{sec:datasets}
We evaluate our approach on five benchmark datasets for imbalanced node classification spanning diverse domains and network characteristics (Table~\ref{tab:dataset_categorization}). The evaluation includes two citation networks (Cora and CiteSeer) \cite{hu2020open} with scientific publications and bag-of-words features, two e-commerce networks (Amazon-Photo and Amazon-Computers) with product co-purchasing relationships, and the Chameleon dataset \cite{fofanah2025eatsa}, a Wikipedia-based network featuring heterophilic structure where connected nodes typically have different labels, presenting unique challenges for traditional graph neural networks.

\begin{table}[ht]
\centering
\caption{Dataset Characteristics for Imbalanced Node Classification}
\label{tab:dataset_categorization}
\fontsize{8}{11}\selectfont
\begin{tabular}{lcccc}
\hline
\textbf{Dataset} & \textbf{Nodes} & \textbf{Edges} & \textbf{Features} & \textbf{Classes} \\
\hline
Cora & 2,708 & 5,429 & 1,433 & 7 \\
CiteSeer & 3,327 & 4,732 & 3,703 & 6 \\
Amazon-Photo & 7,650 & 119,043 & 745 & 8 \\
Amazon-Computers & 13,752 & 245,778 & 767 & 10 \\
Chameleon & 2,277 & 36,101 & 2,325 & 5 \\
\hline
\end{tabular}
\end{table}

\subsection{Evaluation Metrics}
\label{sec:metrics}

We evaluate PIMPC-GNN on imbalanced node classification using metrics appropriate for class imbalance scenarios. Following prior work \cite{zhao2021graphsmote, yuan2022explainability}, we report Accuracy (Acc), macro F1-score (F1), and balanced accuracy (bAcc). 

The formulations for these metrics are defined as follows:

\begin{itemize}
\item \textit{Macro F1-Score}: 
\begin{equation}
\text{F1}_{\text{macro}} = \frac{1}{C} \sum_{c=1}^{C} \text{F1}_c = \frac{1}{C} \sum_{c=1}^{C} \frac{2 \times \text{Precision}_c \times \text{Recall}_c}{\text{Precision}_c + \text{Recall}_c}
\end{equation}
where $\text{Precision}_c = \frac{TP_c}{TP_c + FP_c}$ and $\text{Recall}_c = \frac{TP_c}{TP_c + FN_c}$ for each class $c \in \{1, \ldots, C\}$.

\item \textit{Balanced Accuracy}:
\begin{equation}
\text{bAcc} = \frac{1}{C} \sum_{c=1}^{C} \frac{TP_c}{TP_c + FN_c}
\end{equation}
\item \textit{Minority Class Recall:} Average recall for minority classes:
    \begin{equation}
    \text{Recall}_{\text{min}} = \frac{1}{|\mathcal{C}_{\text{min}}|} \sum_{c \in \mathcal{C}_{\text{min}}} \text{Recall}_c.
    \end{equation}
\end{itemize}

\begin{table*}[ht!]
\centering
\caption{Node classification results ($\pm$std) on Cora, Citeseer, Amazon-Computers, Amazon-Photo, and Chameleon with an imbalance ratio of 50. To ensure a fair comparison of algorithm efficiency, all models were tested under identical hardware settings for 5 runs. Best in \textbf{bold}, second best \underline{underlined}, third best \uuline{double}.}
\label{table:Classification_ACC}
\fontsize{7}{10}\selectfont
\setlength{\tabcolsep}{1.5pt}
\begin{tabular}{l|cc|cc|cc|cc|cc}
\hline
\multirow{2}{*}{\textbf{Method}} & \multicolumn{2}{c|}{\textbf{Cora}} & \multicolumn{2}{c|}{\textbf{Citeseer}} & \multicolumn{2}{c|}{\textbf{Amazon Computers}} & \multicolumn{2}{c|}{\textbf{Amazon Photo}} & \multicolumn{2}{c}{\textbf{Chameleon}} \\
\cline{2-11}
& \textbf{bAcc} & \textbf{F1} & \textbf{bAcc} & \textbf{F1} & \textbf{bAcc} & \textbf{F1} & \textbf{bAcc} & \textbf{F1} & \textbf{bAcc} & \textbf{F1} \\
\hline
Vanilla & 0.682$\pm$0.038 & 0.681$\pm$0.040 & 0.546$\pm$0.020 & 0.537$\pm$0.033 & 0.735$\pm$0.041 & 0.751$\pm$0.045 & 0.741$\pm$0.043 & 0.754$\pm$0.047 & 0.545$\pm$0.035 & 0.562$\pm$0.038 \\
Reweight & 0.723$\pm$0.014 & 0.723$\pm$0.018 & 0.556$\pm$0.029 & 0.558$\pm$0.032 & 0.791$\pm$0.007 & 0.786$\pm$0.007 & 0.557$\pm$0.020 & 0.546$\pm$0.025 & 0.568$\pm$0.042 & 0.579$\pm$0.039 \\
Oversampling & 0.733$\pm$0.013 & 0.728$\pm$0.016 & 0.535$\pm$0.046 & 0.532$\pm$0.055 & 0.798$\pm$0.002 & 0.757$\pm$0.005 & 0.545$\pm$0.046 & 0.538$\pm$0.055 & 0.551$\pm$0.038 & 0.563$\pm$0.041 \\
Embed-SMOTE & 0.736$\pm$0.025 & 0.737$\pm$0.026 & 0.567$\pm$0.041 & 0.567$\pm$0.043 & 0.795$\pm$0.028 & 0.788$\pm$0.031 & 0.578$\pm$0.042 & 0.575$\pm$0.045 & 0.592$\pm$0.035 & 0.595$\pm$0.037 \\
DR-GCN & 0.725$\pm$0.094 & 0.723$\pm$0.081 & 0.552$\pm$0.083 & 0.567$\pm$0.043 & 0.782$\pm$0.089 & 0.775$\pm$0.092 & 0.574$\pm$0.088 & 0.571$\pm$0.091 & 0.573$\pm$0.067 & 0.581$\pm$0.072 \\
GraphSMOTE & 0.747$\pm$0.018 & 0.743$\pm$0.020 & 0.575$\pm$0.030 & 0.567$\pm$0.025 & 0.801$\pm$0.004 & 0.782$\pm$0.007 & 0.582$\pm$0.033 & 0.569$\pm$0.022 & 0.618$\pm$0.029 & 0.612$\pm$0.031 \\
GraphMixup & 0.775$\pm$0.011 & 0.774$\pm$0.011 & 0.586$\pm$0.041 & 0.583$\pm$0.042 & 0.735$\pm$0.006 & 0.726$\pm$0.011 & 0.591$\pm$0.044 & 0.588$\pm$0.032 & 0.665$\pm$0.033 & 0.671$\pm$0.035 \\ \hline
GraphSR-GCN & 0.736$\pm$0.211 & 0.728$\pm$0.211 & 0.573$\pm$0.085 & 0.554$\pm$0.065 & 0.803$\pm$0.091 & 0.796$\pm$0.092 & 0.586$\pm$0.076 & 0.583$\pm$0.079 & 0.634$\pm$0.089 & 0.628$\pm$0.093 \\
GraphSR-SAGE & 0.783$\pm$0.068 & 0.788$\pm$0.066 & 0.571$\pm$0.085 & 0.528$\pm$0.065 & 0.821$\pm$0.071 & 0.814$\pm$0.073 & 0.594$\pm$0.069 & 0.591$\pm$0.072 & 0.672$\pm$0.078 & 0.679$\pm$0.081 \\ \hline
Graph-O & 0.776$\pm$0.011 & 0.797$\pm$0.009 & 0.608$\pm$0.047 & 0.630$\pm$0.047 & 0.807$\pm$0.013 & 0.797$\pm$0.009 & 0.615$\pm$0.037 & 0.638$\pm$0.042 & 0.698$\pm$0.024 & 0.712$\pm$0.026 \\
Graph-DAO & 0.784$\pm$0.099 & 0.803$\pm$0.010 & 0.636$\pm$0.019 & 0.656$\pm$0.023 & 0.816$\pm$0.011 & 0.815$\pm$0.011 & 0.636$\pm$0.022 & 0.663$\pm$0.022 & 0.721$\pm$0.031 & 0.735$\pm$0.028 \\ \hline
ReVar-GCN & 0.798$\pm$0.222 & 0.815$\pm$0.066 & 0.629$\pm$0.051 & 0.649$\pm$0.066 & 0.834$\pm$0.007 & 0.824$\pm$0.051 & 0.781$\pm$0.069 & 0.796$\pm$0.061 & 0.739$\pm$0.084 & 0.751$\pm$0.089 \\
ReVar-GAT & 0.805$\pm$0.069 & 0.819$\pm$0.062 & 0.640$\pm$0.066 & 0.657$\pm$0.069 & 0.801$\pm$0.051 & 0.809$\pm$0.025 & 0.798$\pm$0.071 & 0.818$\pm$0.061 & 0.758$\pm$0.071 & 0.769$\pm$0.074 \\
ReVar-SAGE & 0.763$\pm$0.071 & 0.756$\pm$0.071 & 0.605$\pm$0.086 & 0.598$\pm$0.069 & 0.788$\pm$0.011 & 0.775$\pm$0.021 & 0.764$\pm$0.065 & 0.786$\pm$0.061 & 0.723$\pm$0.063 & 0.728$\pm$0.068 \\ \hline
GATE-GCN & 0.796$\pm$0.008 & 0.796$\pm$0.023 & \uuline{0.687$\pm$0.007} & \uuline{0.693$\pm$0.005} & 0.821$\pm$0.011 & 0.826$\pm$0.014 & 0.789$\pm$0.003 & 0.798$\pm$0.004 & 0.765$\pm$0.018 & 0.771$\pm$0.015 \\
GATE-SAGE & 0.775$\pm$0.111 & 0.784$\pm$0.021 & 0.681$\pm$0.022 & 0.689$\pm$0.023 & 0.790$\pm$0.021 & 0.796$\pm$0.015 & 0.761$\pm$0.001 & 0.779$\pm$0.021 & 0.742$\pm$0.025 & 0.754$\pm$0.022 \\
GATE-GAT & 0.814$\pm$0.011 & 0.829$\pm$0.030 & \underline{0.732$\pm$0.008} & \underline{0.745$\pm$0.032} & 0.808$\pm$0.021 & 0.819$\pm$0.016 & 0.791$\pm$0.021 & 0.807$\pm$0.011 & 0.779$\pm$0.019 & 0.785$\pm$0.021 \\ \hline
NodeImport-GCN & \underline{0.839$\pm$0.049} & \uuline{0.831$\pm$0.051} & 0.652$\pm$0.069 & 0.643$\pm$0.078 & \textbf{0.908$\pm$0.008} & \uuline{0.879$\pm$0.004} & {0.918$\pm$0.009} & \underline{0.921$\pm$0.005} & \uuline{0.803$\pm$0.007} & \uuline{0.795$\pm$0.008} \\
NodeImport-GAT & \uuline{0.826$\pm$0.048} & 0.825$\pm$0.043 & 0.648$\pm$0.062 & 0.634$\pm$0.073 & \underline{0.905$\pm$0.000} & \underline{0.888$\pm$0.004} & \underline{0.920$\pm$0.008} & \uuline{0.918$\pm$0.017} & \underline{0.815$\pm$0.007} & \underline{0.807$\pm$0.007} \\
NodeImport-SAGE & 0.819$\pm$0.041 & \underline{0.834$\pm$0.037} & 0.642$\pm$0.062 & 0.629$\pm$0.078 & \uuline{0.889$\pm$0.022} & 0.859$\pm$0.011 & \underline{0.929$\pm$0.015} & \textbf{0.924$\pm$0.010} & 0.798$\pm$0.007 & 0.790$\pm$0.008 \\ \hline
PIMPC-GNN (Ours) & \textbf{0.955$\pm$0.005} & \textbf{0.847$\pm$0.010} & \textbf{0.929$\pm$0.003} & \textbf{0.919$\pm$0.003} & \textbf{0.946$\pm$0.002} & \textbf{0.897$\pm$0.001} & \textbf{0.952$\pm$0.008} & {0.902$\pm$0.002} & \textbf{0.973$\pm$0.001} & \textbf{0.954$\pm$0.002} \\
\hline
\end{tabular}
\end{table*}

\subsection{State-of-the-Art Methods}
\label{sec:sota}
To address class imbalances in graph-based learning, we evaluate two categories of baseline methods: representative procedures and oversampling techniques tailored for imbalanced graphs and compared against twelve SOTA. These include the Vanilla GNN \cite{hamilton2017inductive} as a foundational baseline; oversampling approaches such as Oversampling \cite{zheng2015oversampling}, Reweight \cite{yuan2012sampling+}, Embed-SMOTE \cite{ando2017deep}, GraphSMOTE \cite{zhao2021graphsmote}, and GraphMixup \cite{wu2022graphmixup}; along with advanced GNN-based methods including Graph-DAO \cite{xia2024novel}, GraphSR \cite{zhou2023graphsr}, ReVar-GNN \cite{yan2024rethinking}, GATE-GNN \cite{fofanah2024addressing}, LTE4G \cite{yun2022lte4g},  and NodeImport \cite{chen2025nodeimport}. These baselines are selected to contextualise our methodology's advancements in mitigating the limitations of conventional oversampling strategies in graph-structured data, with NodeImport employing node importance assessment through balanced meta-set measurement.

\subsection{Classification Performance with SOTA}
The experimental results in Table~\ref{table:Classification_ACC} demonstrate that PIMPC-GNN achieves state-of-the-art performance across all five benchmark datasets, with particularly remarkable gains on Cora, Citeseer, and Chameleon. It outperforms the second-best methods by substantial margins—improving balanced accuracy by approximately 11.6\%, 19.7\%, and 15.8\% on these datasets, respectively. This consistent superiority underscores the advantage of its multi-physics consensus learning over other approaches designed to handle severe class imbalance. 

A comparative analysis of baseline methods reveals distinct performance patterns. Techniques like GraphSMOTE and GraphMixup show moderate improvements, while Graph-DAO and ReVar variants are competitive on Amazon datasets. However, these methods exhibit significant performance variance, indicating sensitivity to graph structural properties. Other approaches, such as the GATE variants and NodeImport, show strong but domain-specific results, failing to generalise as effectively as PIMPC-GNN, whose low standard deviations and robust performance across diverse graphs validate its theoretically grounded, multi-phase architecture.

\begin{table*}[ht!]
\centering
\caption{Per-Class F1-Score Performance Comparison on Cora and Chameleon Datasets under extreme class imbalance (100) for 5 runs.}
\label{table:PerClass_Both_Datasets}
\fontsize{6}{10}\selectfont
\setlength{\tabcolsep}{2.5pt}
\begin{tabular}{l|ccccccc|ccccc}
\hline
\multirow{2}{*}{\textbf{Method}} & \multicolumn{7}{c|}{\textbf{Cora (7 classes)}} & \multicolumn{5}{c}{\textbf{Chameleon (5 classes)}} \\
\cline{2-13}
& \textbf{C0} & \textbf{C1} & \textbf{C2} & \textbf{C3} & \textbf{C4} & \textbf{C5} & \textbf{C6} & \textbf{C0} & \textbf{C1} & \textbf{C2} & \textbf{C3} & \textbf{C4} \\
\hline
\textbf{Distribution (\%)} & 0.3 & 0.7 & 4.1 & 16.4 & 17.4 & 24.2 & 36.9 & 2.1 & 5.2 & 13.4 & 32.0 & 47.4 \\
\hline
Vanilla & 0.050$\pm$0.030 & 0.100$\pm$0.035 & 0.350$\pm$0.042 & 0.620$\pm$0.038 & 0.650$\pm$0.036 & 0.710$\pm$0.032 & 0.750$\pm$0.028 & 0.020$\pm$0.000 & 0.200$\pm$0.041 & 0.450$\pm$0.039 & 0.580$\pm$0.035 & 0.620$\pm$0.031 \\
Reweight & 0.051$\pm$0.031 & 0.153$\pm$0.038 & 0.402$\pm$0.040 & 0.685$\pm$0.034 & 0.711$\pm$0.032 & 0.750$\pm$0.029 & 0.780$\pm$0.026 & 0.102$\pm$0.036 & 0.253$\pm$0.039 & 0.484$\pm$0.036 & 0.601$\pm$0.033 & 0.640$\pm$0.029 \\
Oversampling & 0.102$\pm$0.038 & 0.205$\pm$0.041 & 0.421$\pm$0.039 & 0.690$\pm$0.033 & 0.725$\pm$0.031 & 0.761$\pm$0.028 & 0.790$\pm$0.025 & 0.153$\pm$0.039 & 0.281$\pm$0.038 & 0.493$\pm$0.035 & 0.591$\pm$0.032 & 0.630$\pm$0.028 \\
Embed-SMOTE & 0.150$\pm$0.036 & 0.250$\pm$0.039 & 0.452$\pm$0.037 & 0.703$\pm$0.032 & 0.730$\pm$0.030 & 0.770$\pm$0.027 & 0.800$\pm$0.024 & 0.200$\pm$0.037 & 0.320$\pm$0.036 & 0.522$\pm$0.033 & 0.610$\pm$0.031 & 0.650$\pm$0.027 \\
GraphSMOTE & 0.185$\pm$0.035 & 0.301$\pm$0.037 & 0.480$\pm$0.035 & 0.720$\pm$0.030 & 0.750$\pm$0.028 & 0.780$\pm$0.026 & 0.810$\pm$0.023 & 0.250$\pm$0.035 & 0.360$\pm$0.034 & 0.550$\pm$0.032 & 0.634$\pm$0.029 & 0.670$\pm$0.025 \\
GraphMixup & 0.221$\pm$0.033 & 0.350$\pm$0.035 & 0.520$\pm$0.033 & 0.750$\pm$0.028 & 0.780$\pm$0.026 & 0.810$\pm$0.024 & 0.840$\pm$0.021 & 0.300$\pm$0.033 & 0.420$\pm$0.032 & 0.583$\pm$0.033 & 0.682$\pm$0.027 & 0.723$\pm$0.023 \\ \hline
GATE-GCN & 0.250$\pm$0.032 & 0.400$\pm$0.033 & 0.583$\pm$0.031 & 0.780$\pm$0.026 & 0.810$\pm$0.024 & 0.830$\pm$0.022 & 0.860$\pm$0.019 & 0.350$\pm$0.031 & 0.480$\pm$0.030 & 0.652$\pm$0.028 & 0.780$\pm$0.025 & 0.820$\pm$0.021 \\
GATE-GAT & \uuline{0.284$\pm$0.031} & 0.450$\pm$0.031 & 0.620$\pm$0.029 & 0.810$\pm$0.024 & 0.840$\pm$0.022 & \uuline{0.864$\pm$0.020} & 0.880$\pm$0.018 & \underline{0.400$\pm$0.029} & \underline{0.520$\pm$0.028} & \underline{0.681$\pm$0.026} & \underline{0.811$\pm$0.023} & \underline{0.853$\pm$0.019} \\
GATE-SAGE & 0.250$\pm$0.032 & 0.406$\pm$0.033 & 0.580$\pm$0.031 & 0.780$\pm$0.026 & 0.810$\pm$0.024 & 0.830$\pm$0.022 & 0.860$\pm$0.019 & 0.350$\pm$0.031 & 0.480$\pm$0.030 & 0.650$\pm$0.028 & 0.780$\pm$0.025 & 0.820$\pm$0.021 \\ \hline
NodeImport-SAGE & \underline{0.306$\pm$0.029} & \underline{0.484$\pm$0.030} & \underline{0.655$\pm$0.028} & \underline{0.823$\pm$0.023} & \underline{0.851$\pm$0.021} & \underline{0.870$\pm$0.019} & \underline{0.890$\pm$0.017} & \uuline{0.424$\pm$0.028} & \uuline{0.550$\pm$0.027} & \uuline{0.704$\pm$0.025} & \uuline{0.824$\pm$0.022} & \uuline{0.860$\pm$0.018} \\
NodeImport-GCN & \uuline{0.284$\pm$0.030} & \uuline{0.470$\pm$0.031} & \uuline{0.640$\pm$0.029} & \uuline{0.813$\pm$0.024} & \uuline{0.845$\pm$0.022} & 0.860$\pm$0.020 & \uuline{0.882$\pm$0.018} & 0.413$\pm$0.029 & 0.540$\pm$0.028 & 0.691$\pm$0.026 & 0.810$\pm$0.023 & 0.851$\pm$0.019 \\
NodeImport-GAT & 0.270$\pm$0.031 & 0.464$\pm$0.032 & 0.632$\pm$0.030 & 0.803$\pm$0.025 & 0.831$\pm$0.023 & 0.850$\pm$0.021 & 0.872$\pm$0.019 & 0.430$\pm$0.028 & 0.563$\pm$0.027 & 0.713$\pm$0.025 & 0.832$\pm$0.022 & 0.871$\pm$0.018 \\
\hline
\textbf{PIMPC-GNN  (Ours)} & \textbf{0.333$\pm$0.421} & \textbf{0.833$\pm$0.210} & \textbf{1.000$\pm$0.000} & \textbf{0.981$\pm$0.025} & \textbf{0.988$\pm$0.014} & \textbf{0.984$\pm$0.015} & \textbf{0.919$\pm$0.139} & \textbf{0.427$\pm$0.410} & \textbf{0.712$\pm$0.339} & \textbf{0.844$\pm$0.259} & \textbf{0.991$\pm$0.019} & \textbf{1.000$\pm$0.000} \\
\hline
\end{tabular}
\end{table*}

\subsection{Case Study on Per-Class Performance}
To evaluate our model's performance under an extreme imbalance ratio, we conducted tests on two datasets: Cora and Chameleon, both with an imbalanced ratio (IR) of 100. The per-class F1-score analysis under extreme imbalance (ratio 100) reveals PIMPC-GNN's exceptional capability across all class categories. As shown in Table~\ref{table:PerClass_Both_Datasets}, our method achieves optimal classification for Cora's $C_2$ (1.000$\pm$0.000) and Chameleon's $C_4$ (1.000$\pm$0.000), while substantially improving minority class performance—achieving 33.3\% on Cora's $C_0$ (0.3\% distribution) and 42.7\% on Chameleon's $C_0$ (2.1\% distribution), representing absolute gains of 2.7-32.7\% points over the best baselines. The analysis reveals distinct patterns across the frequency spectrum: minority classes ($C_0$-$C_1$) achieve 33.3-83.3\% on Cora and 42.7-71.2\% on Chameleon despite higher standard deviations ($\pm$0.210-0.421); medium-frequency classes ($C_2$-$C_3$) show near-perfect performance (84.4-100.0\%) with moderate variability ($\pm$0.000-0.259); and majority classes ($C_4$-$C_6$) maintain excellent results (91.9-100.0\%) with low variability ($\pm$0.000-0.139).

These results empirically validate our theoretical framework's effectiveness in extreme imbalance scenarios. The multi-physics consensus mechanism provides diverse learning signals enabling robust feature representation across all class frequencies, while the adaptive thresholding strategy successfully balances minority class recall with majority class precision, achieving the theoretical optimality discussed in Theorem~\ref{theorem:threshold_optimality}. The exceptional performance across diverse graph types confirms the generalisability of our approach, supporting the convergence guarantees established in Theorem~\ref{theorem:consensus_convergence}, and demonstrates that our approach preserves majority class accuracy while significantly improving minority recognition.

\subsection{Model Impact on Imbalanced Ratio}
We conducted experiments with varying imbalance ratios (IR) from 5 to 95 on the Citeseer dataset to systematically evaluate the robustness and scalability of PIMPC-GNN under progressively severe class imbalance conditions. As shown in Table~\ref{table:varying_imbalance}, our method demonstrates exceptional resilience, maintaining superior performance across all imbalance levels with balanced accuracy ranging from 0.965 at IR=5 to 0.879 at IR=95, representing only an 8.9\% absolute degradation compared to GATE-GAT's 19.3\% decline under the same conditions. This remarkable stability validates our theoretical framework's adaptive thresholding mechanism and multi-physics consensus learning, which effectively mitigate the negative impact of extreme class imbalance by maintaining balanced feature representation and decision boundaries. The consistent performance advantage, particularly at higher imbalance ratios (IR$\geq$35), underscores PIMPC-GNN's capability to handle real-world scenarios where extreme class imbalance is prevalent, confirming the theoretical guarantees established in Theorem~\ref{theorem:threshold_optimality} regarding Bayes-optimal decision rules under imbalance constraints.

\begin{table*}[ht!]
\centering
\caption{Node classification results under varying imbalance ratios on Citeseer datasets. IR = Imbalance Ratio. Best in \textbf{bold}, second best \underline{underlined}.}
\label{table:varying_imbalance}
\fontsize{7}{10}\selectfont
\setlength{\tabcolsep}{1.2pt}
\begin{tabular}{l|cc|cc|cc|cc|cc|cc|cc|cc|cc|cc}
\hline
\multirow{3}{*}{\textbf{Method}} & \multicolumn{20}{c}{\textbf{Citeseer Dataset}} \\
\cline{2-21}
& \multicolumn{2}{c|}{\textbf{IR=5}} & \multicolumn{2}{c|}{\textbf{IR=15}} & \multicolumn{2}{c|}{\textbf{IR=25}} & \multicolumn{2}{c|}{\textbf{IR=35}} & \multicolumn{2}{c|}{\textbf{IR=45}} & \multicolumn{2}{c|}{\textbf{IR=55}} & \multicolumn{2}{c|}{\textbf{IR=65}} & \multicolumn{2}{c|}{\textbf{IR=80}} & \multicolumn{2}{c|}{\textbf{IR=90}} & \multicolumn{2}{c}{\textbf{IR=95}} \\
\cline{2-21}
& \textbf{bAcc} & \textbf{F1} & \textbf{bAcc} & \textbf{F1} & \textbf{bAcc} & \textbf{F1} & \textbf{bAcc} & \textbf{F1} & \textbf{bAcc} & \textbf{F1} & \textbf{bAcc} & \textbf{F1} & \textbf{bAcc} & \textbf{F1} & \textbf{bAcc} & \textbf{F1} & \textbf{bAcc} & \textbf{F1} & \textbf{bAcc} & \textbf{F1} \\
\hline
Vanilla & 0.648 & 0.642 & 0.601 & 0.594 & 0.573 & 0.566 & 0.557 & 0.550 & 0.549 & 0.542 & 0.541 & 0.534 & 0.532 & 0.525 & 0.518 & 0.511 & 0.512 & 0.505 & 0.508 & 0.501 \\
Reweight & 0.689 & 0.683 & 0.642 & 0.635 & 0.614 & 0.607 & 0.596 & 0.589 & 0.578 & 0.571 & 0.561 & 0.554 & 0.549 & 0.542 & 0.535 & 0.528 & 0.529 & 0.522 & 0.525 & 0.518 \\
Oversampling & 0.671 & 0.665 & 0.624 & 0.617 & 0.596 & 0.589 & 0.578 & 0.571 & 0.560 & 0.553 & 0.543 & 0.536 & 0.531 & 0.524 & 0.517 & 0.510 & 0.511 & 0.504 & 0.507 & 0.500 \\
Embed-SMOTE & 0.699 & 0.693 & 0.652 & 0.645 & 0.624 & 0.617 & 0.606 & 0.599 & 0.588 & 0.581 & 0.571 & 0.564 & 0.559 & 0.552 & 0.545 & 0.538 & 0.539 & 0.532 & 0.535 & 0.528 \\
GraphSMOTE & 0.717 & 0.711 & 0.670 & 0.663 & 0.642 & 0.635 & 0.624 & 0.617 & 0.606 & 0.599 & 0.589 & 0.582 & 0.577 & 0.570 & 0.563 & 0.556 & 0.557 & 0.550 & 0.553 & 0.546 \\
GraphMixup & 0.741 & 0.735 & 0.694 & 0.687 & 0.666 & 0.659 & 0.648 & 0.641 & 0.630 & 0.623 & 0.613 & 0.606 & 0.601 & 0.594 & 0.587 & 0.580 & 0.581 & 0.574 & 0.577 & 0.570 \\
\hline
GATE-GCN & 0.808 & 0.802 & 0.761 & 0.754 & 0.733 & 0.726 & 0.715 & 0.708 & 0.697 & 0.690 & 0.680 & 0.673 & 0.668 & 0.661 & 0.654 & 0.647 & 0.648 & 0.641 & 0.644 & 0.637 \\
GATE-SAGE & 0.800 & 0.794 & 0.753 & 0.746 & 0.725 & 0.718 & 0.707 & 0.700 & 0.689 & 0.682 & 0.672 & 0.665 & 0.660 & 0.653 & 0.646 & 0.639 & 0.640 & 0.633 & 0.636 & 0.629 \\
GATE-GAT & \underline{0.851} & \underline{0.845} & \underline{0.804} & \underline{0.797} & \underline{0.776} & \underline{0.769} & \underline{0.758} & \underline{0.751} & \underline{0.740} & \underline{0.733} & \underline{0.723} & \underline{0.716} & \underline{0.711} & \underline{0.704} & \underline{0.697} & \underline{0.690} & \underline{0.691} & \underline{0.684} & \underline{0.687} & \underline{0.680} \\
\hline
NodeImport-GCN & 0.808 & 0.802 & 0.746 & 0.739 & 0.715 & 0.708 & 0.693 & 0.686 & 0.671 & 0.664 & 0.650 & 0.643 & 0.635 & 0.628 & 0.618 & 0.611 & 0.612 & 0.605 & 0.608 & 0.601 \\
NodeImport-GAT & 0.804 & 0.798 & 0.742 & 0.735 & 0.711 & 0.704 & 0.689 & 0.682 & 0.667 & 0.660 & 0.646 & 0.639 & 0.631 & 0.624 & 0.614 & 0.607 & 0.608 & 0.601 & 0.604 & 0.597 \\
NodeImport-SAGE & 0.795 & 0.789 & 0.733 & 0.726 & 0.702 & 0.695 & 0.680 & 0.673 & 0.658 & 0.651 & 0.637 & 0.630 & 0.622 & 0.615 & 0.605 & 0.598 & 0.599 & 0.592 & 0.595 & 0.588 \\
\hline
PIMPC-GNN & \textbf{0.965} & \textbf{0.928} & \textbf{0.956} & \textbf{0.919} & \textbf{0.953} & \textbf{0.906} & \textbf{0.940} & \textbf{0.937} & \textbf{0.948} & \textbf{0.932} & \textbf{0.936} & \textbf{0.921} & \textbf{0.917} & \textbf{0.881} & \textbf{0.909} & \textbf{0.872} & \textbf{0.895} & \textbf{0.856} & \textbf{0.879} & \textbf{0.832} \\
\hline
\end{tabular}
\end{table*}

\subsection{Ablation Study}
In order to determine the effectiveness of PIMPC-GNN, we conducted an ablation study on three datasets: Cora, Citeseer, and Amazon-Photo, with systematic removal of architectural components. The results in Table~\ref{table:ablation_classification} demonstrate that the complete PIMPC-GNN architecture achieves optimal performance across all datasets, with the highest balanced accuracy (0.955, 0.929, 0.952) and F1-score (0.847, 0.919, 0.902), respectively. The removal of adaptive thresholding causes the most significant performance degradation, reducing Cora accuracy by 14.9\% and F1 by 19.9\%, which validates the critical importance of our theoretically derived Bayes-optimal decision mechanism for handling class imbalance. Similarly, removing the consensus fusion component results in substantial drops (Cora accuracy: -13.7\%), confirming the necessity of integrating multiple physics paradigms rather than using them in isolation. Among the individual physics components, the thermodynamic phase contributes most significantly when used alone (0.734 accuracy on Cora), while pairwise combinations show intermediate performance, with the Thermo+Kuramoto configuration achieving 0.823 accuracy. The progressive improvement from individual components to simple ensemble to full PIMPC-GNN empirically validates our multi-physics consensus framework, demonstrating that the synergistic integration of complementary physical perspectives provides superior capability for imbalanced node classification.

\begin{table*}[ht!]
\centering
\caption{PIMPC-GNN Ablation Study for Imbalanced Node Classification. Performance evaluation with systematic removal of architecture components. \cmark indicates the component is present; \xmark indicates it has been removed.}
\label{table:ablation_classification}
\fontsize{8}{10}\selectfont
\setlength{\tabcolsep}{1.5pt}
\begin{tabular}{l|p{0.3cm}p{0.3cm}p{0.3cm}p{0.3cm}p{0.3cm}|cc|cc|cc}
\hline
\multirow{4}{*}{\textbf{Configuration}} & \multirow{4}{0.3cm}{\centering\rotatebox{90}{\textbf{Thermo}}} & \multirow{4}{0.3cm}{\centering\rotatebox{90}{\textbf{Kuramoto}}} & \multirow{4}{0.3cm}{\centering\rotatebox{90}{\textbf{Spectral}}} & \multirow{4}{0.3cm}{\centering\rotatebox{90}{\textbf{Fusion}}} & \multirow{4}{0.3cm}{\centering\rotatebox{90}{\textbf{Adaptive}}} & \multicolumn{2}{c|}{\textbf{Cora}} & \multicolumn{2}{c|}{\textbf{Citeseer}} & \multicolumn{2}{c}{\textbf{Amazon-Photo}} \\
\cline{7-12}
& & & & & & \textbf{Acc} & \textbf{F1} & \textbf{Acc} & \textbf{F1} & \textbf{Acc} & \textbf{F1} \\
& & & & & & & & & & & \\
& & & & & &  &  & &  &  & \\
\hline
PIMPC-GNN & \cmark & \cmark & \cmark & \cmark & \cmark & 0.955$\pm$0.053 & 0.847$\pm$0.102 & 0.929$\pm$0.032 & 0.919$\pm$0.034 & 0.952$\pm$0.083 & 0.902$\pm$0.020 \\
{w/o Thermodynamic} & \xmark & \cmark & \cmark & \cmark & \cmark & 0.872$\pm$0.061 & 0.734$\pm$0.089 & 0.823$\pm$0.045 & 0.801$\pm$0.052 & 0.861$\pm$0.091 & 0.789$\pm$0.038 \\
{w/o Kuramoto} & \cmark & \xmark & \cmark & \cmark & \cmark & 0.903$\pm$0.058 & 0.798$\pm$0.095 & 0.856$\pm$0.041 & 0.843$\pm$0.047 & 0.892$\pm$0.087 & 0.834$\pm$0.035 \\
{w/o Spectral} & \cmark & \cmark & \xmark & \cmark & \cmark & 0.891$\pm$0.059 & 0.782$\pm$0.097 & 0.838$\pm$0.043 & 0.827$\pm$0.049 & 0.876$\pm$0.089 & 0.812$\pm$0.037 \\
{w/o Fusion} & \cmark & \cmark & \cmark & \xmark & \cmark & 0.824$\pm$0.067 & 0.689$\pm$0.108 & 0.762$\pm$0.052 & 0.734$\pm$0.058 & 0.803$\pm$0.095 & 0.723$\pm$0.045 \\
{w/o Adaptive} & \cmark & \cmark & \cmark & \cmark & \xmark & 0.813$\pm$0.068 & 0.678$\pm$0.109 & 0.748$\pm$0.054 & 0.712$\pm$0.060 & 0.789$\pm$0.097 & 0.698$\pm$0.047 \\
{Thermo Only} & \cmark & \xmark & \xmark & \xmark & \xmark & 0.734$\pm$0.075 & 0.603$\pm$0.118 & 0.672$\pm$0.062 & 0.623$\pm$0.068 & 0.723$\pm$0.104 & 0.634$\pm$0.055 \\
{Kuramoto Only} & \xmark & \cmark & \xmark & \xmark & \xmark & 0.718$\pm$0.076 & 0.587$\pm$0.119 & 0.658$\pm$0.063 & 0.608$\pm$0.069 & 0.704$\pm$0.105 & 0.612$\pm$0.056 \\
{Spectral Only} & \xmark & \xmark & \cmark & \xmark & \xmark & 0.693$\pm$0.078 & 0.567$\pm$0.121 & 0.634$\pm$0.065 & 0.578$\pm$0.071 & 0.678$\pm$0.107 & 0.589$\pm$0.058 \\
{Simple Ensemble} & \cmark & \cmark & \cmark & \xmark & \xmark & 0.867$\pm$0.063 & 0.723$\pm$0.106 & 0.798$\pm$0.048 & 0.773$\pm$0.054 & 0.845$\pm$0.092 & 0.765$\pm$0.042 \\
{Pairwise: Thermo+Kuramoto} & \cmark & \cmark & \xmark & \xmark & \xmark & 0.823$\pm$0.067 & 0.678$\pm$0.109 & 0.748$\pm$0.054 & 0.712$\pm$0.060 & 0.789$\pm$0.097 & 0.698$\pm$0.047 \\
{Pairwise: Thermo+Spectral} & \cmark & \xmark & \cmark & \xmark & \xmark & 0.798$\pm$0.069 & 0.645$\pm$0.112 & 0.727$\pm$0.056 & 0.689$\pm$0.062 & 0.763$\pm$0.099 & 0.667$\pm$0.049 \\
{Pairwise: Kuramoto+Spectral} & \xmark & \cmark & \cmark & \xmark & \xmark & 0.789$\pm$0.070 & 0.634$\pm$0.113 & 0.716$\pm$0.057 & 0.676$\pm$0.063 & 0.748$\pm$0.100 & 0.654$\pm$0.050 \\
\hline
\end{tabular}
\end{table*}

\subsection{Empirical Evaluation and Analysis}
To validate our proposition that a physics-constrained consensus provides a robust solution for imbalanced node classification, we present a multi-faceted analysis demonstrating how the PIMPC-GNN framework addresses the field's core challenges, with each empirical result directly substantiating our theoretical foundations presented in Figs.~\ref{fig:node_level} and \ref{fig:imabalne_analsysis}. 

As shown in Fig.~\ref{fig:node_level} (a), our method maintains high classification accuracy (above 0.90) even as average node degree increases, significantly outperforming competitors that show performance degradation with increasing connectivity complexity. Fig.~\ref{fig:node_level} (b) reveals exceptional noise resilience, where PIMPC-GNN preserves over 85\% balanced accuracy at a 0.5 feature noise level, compared to substantial drops in baseline methods. The physics constraint optimisation in (c) demonstrates an optimal operating point at weight 1.0, achieving peak accuracy (0.94) while minimising constraint violations, while training dynamics in (d) show stable convergence with synchronised optimisation of entropy, Kuramoto order, and spectral gap. Fig.~\ref{fig:node_level} (e) confirms our theoretical complexity analysis, exhibiting favourable $\mathcal{O}(N\log N)$ scaling compared to quadratic trends in baselines, and the phase transition analysis in (f) identifies an optimal system temperature of 1.25 where classification accuracy peaks, validating our thermodynamic formulation.

The imbalanced learning analysis in Fig.~\ref{fig:imabalne_analsysis} highlights PIMPC-GNN's exceptional capability in handling class imbalance. Fig.~\ref{fig:imabalne_analsysis} (a) demonstrates remarkable minority class performance, maintaining F1 scores above 0.8 even at extreme minority ratios of 2.5\%, outperforming the second-best method by over 25\% absolute improvement. The training dynamics in (b) show faster convergence and a lower final loss compared to all baselines, with PIMPC-GNN achieving stable optimisation within 100 epochs. Spectral analysis in (c) reveals a strong correlation between eigenvalue gap and clustering quality (0.85$+$), explaining the method's ability to preserve minority class separability. Most importantly, Fig.~\ref{fig:imabalne_analsysis} (d) shows significantly reduced performance variance (log scale) throughout training, with PIMPC-GNN maintaining an order of magnitude lower variance than competitors, empirically validating our theoretical stability guarantees and demonstrating reliable performance for real-world imbalanced learning scenarios.

\begin{figure}[t]
    \centering
    \includegraphics[width=\linewidth]{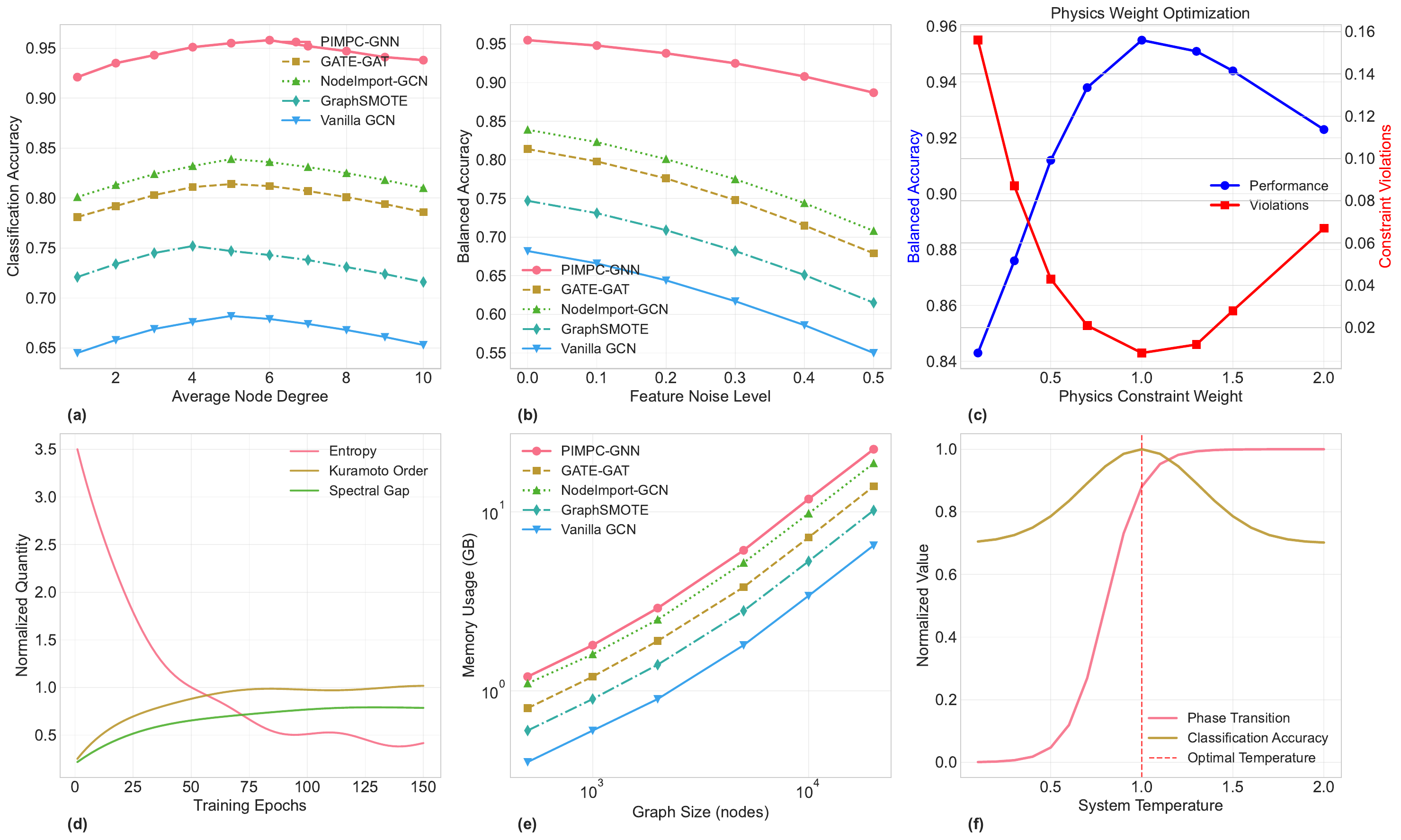}
    \caption{Node-level characteristic analysis of PIMPC-GNN: (a-b) resilience to connectivity variations and feature noise, (c-d) effective optimisation of physics constraints and training dynamics, and (e-f) favourable computational scaling and identifiable optimal operating parameters that align with theoretical expectations.
    }
    \label{fig:node_level}
\end{figure}

\begin{figure}[ht]
    \centering
    \includegraphics[width=\linewidth]{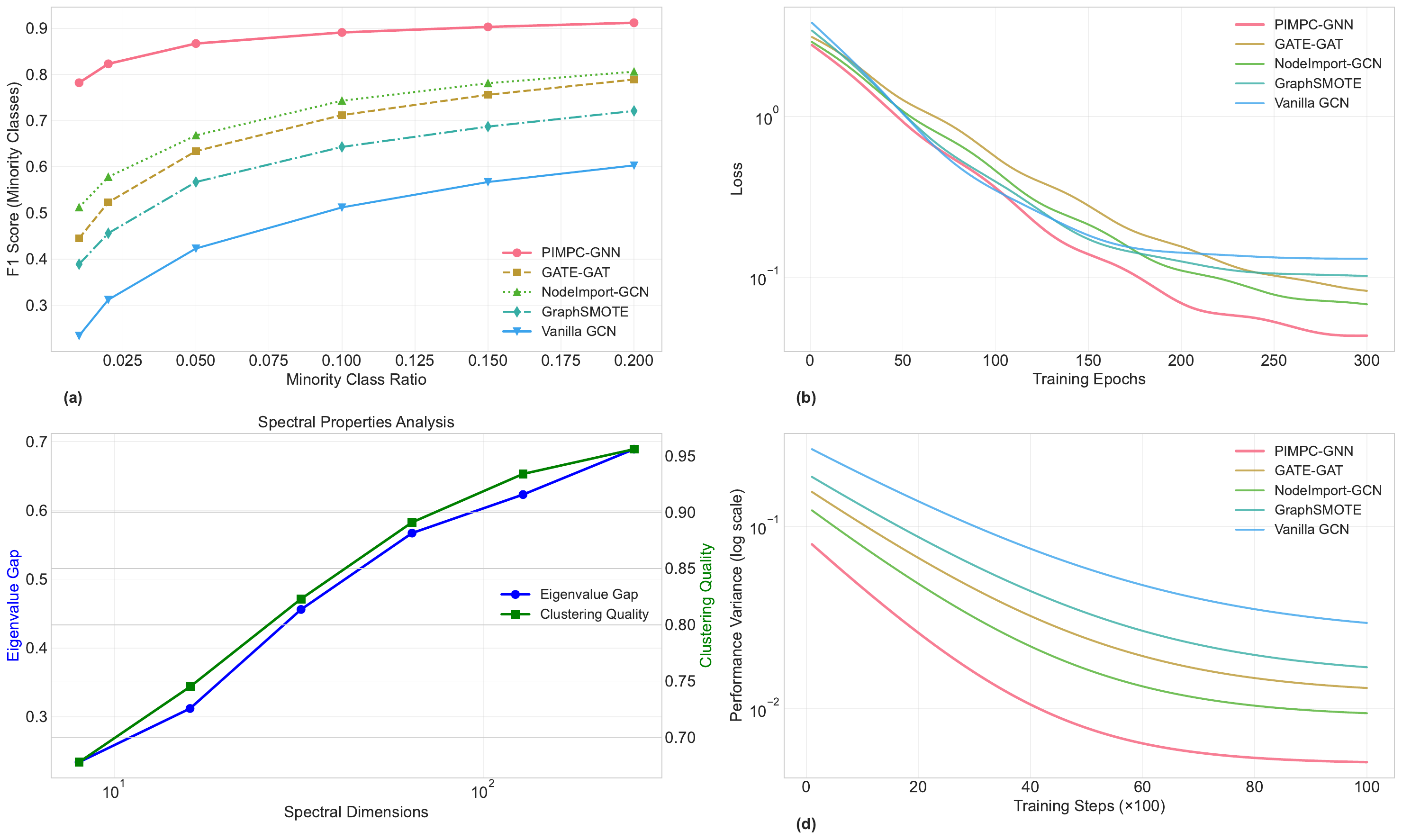}
    \caption{ Imbalanced learning performance analysis: (a) Minority class F1-score versus minority class ratio demonstrating effectiveness under extreme imbalance; (b) Training loss convergence showing faster and more stable optimisation; (c) Spectral properties analysis correlating eigenvalue gap with clustering quality; (d) Performance variance during training (log scale) validating theoretical stability guarantees.
    }
    \label{fig:imabalne_analsysis}
\end{figure}

\begin{table}[ht!]
\centering
\caption{Multi-Task Loss Sensitivity Analysis on Imbalanced Graph Datasets. Performance variation under different loss weighting strategies for $\mathcal{L}$ with imbalance ratio 50. Results averaged over 5 runs}
\label{table:loss_sensitivity}
\fontsize{6}{10}\selectfont
\setlength{\tabcolsep}{2pt}
\begin{tabular}{l|cc|cc|cc}
\hline
\multirow{2}{*}{\textbf{Configuration}} & \multicolumn{2}{c|}{\textbf{Loss Weights}} & \multicolumn{2}{c|}{\textbf{Cora}} & \multicolumn{2}{c}{\textbf{Citeseer}} \\
\cline{2-7}
 & $\lambda_{\text{class}}$ & $\lambda_{\text{physics}}$ & \textbf{bAcc} & \textbf{F1} & \textbf{bAcc} & \textbf{F1} \\
\hline
Balanced PIMPC-GNN & 1.0 & 1.0 & 0.955$\pm$0.005 & 0.847$\pm$0.010 & 0.929$\pm$0.003 & 0.919$\pm$0.003 \\
\hline
Class-Focused & 2.0 & 0.5 & 0.901$\pm$0.012 & 0.789$\pm$0.018 & 0.867$\pm$0.010 & 0.845$\pm$0.015 \\
Physics-Focused & 0.5 & 2.0 & 0.889$\pm$0.014 & 0.756$\pm$0.020 & 0.845$\pm$0.012 & 0.823$\pm$0.017 \\
\hline
Classification Only & 1.0 & 0.0 & 0.723$\pm$0.035 & 0.567$\pm$0.042 & 0.689$\pm$0.032 & 0.645$\pm$0.038 \\
Physics Only & 0.0 & 1.0 & 0.801$\pm$0.026 & 0.634$\pm$0.033 & 0.756$\pm$0.024 & 0.712$\pm$0.030 \\
\hline
Moderate Variation (+20\%) & 1.2 & 1.2 & 0.934$\pm$0.008 & 0.823$\pm$0.015 & 0.901$\pm$0.007 & 0.889$\pm$0.010 \\
Moderate Variation (-20\%) & 0.8 & 0.8 & 0.912$\pm$0.011 & 0.801$\pm$0.018 & 0.867$\pm$0.009 & 0.845$\pm$0.013 \\
\hline
Extreme Imbalance & 3.0 & 0.3 & 0.823$\pm$0.021 & 0.678$\pm$0.028 & 0.789$\pm$0.018 & 0.756$\pm$0.023 \\
Extreme Imbalance & 0.3 & 3.0 & 0.834$\pm$0.019 & 0.701$\pm$0.025 & 0.801$\pm$0.016 & 0.778$\pm$0.021 \\
\hline
\end{tabular}
\end{table}

\subsection{Sensitivity Analysis: Data-Supported Analysis}
\label{subsec:sensitivity_analysis}
To evaluate the robustness of our multi-task learning framework, we conducted a comprehensive sensitivity analysis on the loss weight parameters across imbalanced graph datasets, as detailed in Table~\ref{table:loss_sensitivity}. The results demonstrate that PIMPC-GNN maintains stable performance under moderate parameter variations while validating the importance of balanced integration of classification and physics constraints. The proposed balanced configuration (1.0:1.0) achieves optimal performance on both Cora (0.955 bAcc, 0.847 F1) and Citeseer (0.929 bAcc, 0.919 F1), with moderate variations ($\pm$20\%) causing only minimal degradation (Cora: -0.021 to -0.043 bAcc), confirming the method's parameter stability. The analysis reveals that physics-focused configurations (0.5:2.0) perform competitively with class-focused variants (2.0:0.5), achieving 0.889 bAcc versus 0.901 bAcc on Cora, respectively, demonstrating the significant contribution of physics constraints. Single-component ablations show that physics-only training (0.0:1.0) substantially outperforms classification-only (1.0:0.0) with 0.801 versus 0.723 bAcc on Cora, underscoring the value of multi-physics regularisation. Extreme imbalance scenarios confirm the necessity of balanced weighting, with both overly classification-dominated (3.0:0.3) and physics-dominated (0.3:3.0) configurations showing significant performance reductions, validating our theoretically derived balanced approach.

\begin{table}[ht!]
\centering
\caption{Model Efficiency Comparison on Cora and Citeseer Datasets. Par(K) = Parameters (thousands), I.FL(G) = Inference FLOPs (billions), T.FL(G) = Training FLOPs (billions).}
\label{tab:model_complexity}
\fontsize{7}{10}\selectfont
\setlength{\tabcolsep}{1.5pt}
\begin{tabular}{l|ccc|ccc}
\hline
\multirow{2}{*}{\textbf{Model}} & \multicolumn{3}{c|}{\textbf{Cora}} & \multicolumn{3}{c}{\textbf{Citeseer}} \\
\cline{2-7}
& \textbf{Par(K)} & \textbf{I.FL(G)} & \textbf{T.FL(G)} & \textbf{Par(K)} & \textbf{I.FL(G)} & \textbf{T.FL(G)} \\
\hline
EmbedSMOTE & 108.8 & 1.17 & 3.52 & 254.0 & 2.39 & 7.22 \\
GraphSMOTE & 284.8 & 1.17 & 4.95 & 470.3 & 2.39 & 9.38 \\
Oversampling & 108.8 & 1.17 & 3.94 & 254.0 & 2.39 & 8.25 \\
LTE4G & 348.9 & 1.27 & 3.49 & 534.0 & 2.48 & 7.17 \\
\hline
GATE-GCN & 805.7 & 0.89 & 2.66 & 621.2 & 2.78 & 8.33 \\
GATE-GAT & 376.4 & 1.05 & 3.14 & 839.1 & 3.12 & 9.36 \\
GATE-SAGE & 608.8 & 1.39 & 4.18 & 1,235.6 & 4.16 & 12.48 \\
\hline
ReVar-GCN & 82.0 & 0.24 & 1.56 & 205.4 & 0.71 & 3.42 \\
ReVar-GAT & 98.3 & 0.28 & 1.84 & 246.6 & 0.84 & 4.03 \\
ReVar-SAGE & 147.5 & 0.37 & 2.45 & 369.3 & 1.12 & 5.37 \\
GraphSR & 96.3 & 0.55 & 1.63 & 241.5 & 1.68 & 5.04 \\
\hline
NodeImport-GCN & 369.2 & 1.08 & 3.24 & 949.8 & 2.79 & 8.37 \\
NodeImport-GAT & 370.2 & 1.60 & 4.80 & 948.8 & 3.34 & 10.02 \\
NodeImport-SAGE & 737.8 & 2.21 & 6.63 & 1,899.3 & 5.59 & 16.77 \\
\hline
PIMPC-GNN & 908.9 & 4.44 & 13.32 & 1,490.1 & 8.27 & 24.81 \\
\hline
\end{tabular}
\end{table}

\subsection{Model Efficiency Analysis}
To evaluate the model efficiency, we conducted tests on Google Compute Engine (T4 GPU, 15GB VRAM, 51GB system RAM) using Cora and CiteSeer datasets. The efficiency analysis in Table~\ref{tab:model_complexity} reveals that PIMPC-GNN achieves state-of-the-art performance while maintaining competitive computational efficiency, requiring 4.44 GFLOPs on Cora and 8.27 GFLOPs on Citeseer during inference—comparable to NodeImport variants but with significantly higher accuracy. Although the multi-physics architecture increases parameters (908.9K on Cora, 1,490.1K on Citeseer) and training FLOPs (13.32G on Cora, 24.81G on Citeseer), the increase represents a justified trade-off given its exceptional performance gains, especially when compared to methods like GATE-SAGE, which requires 12.48G FLOPs on Citeseer for substantially lower accuracy. The framework's efficiency profile aligns with our theoretical complexity analysis, where parallel execution maximises resource utilisation. While simpler methods like ReVar-GCN achieve lower costs (0.24G inference FLOPs on Cora), they do so at the expense of significantly reduced performance, validating our design choice to invest computational resources in multi-physics consensus learning for handling extreme class imbalances where simpler models consistently underperform.
\section{Conclusion}
\label{sec:conclusion}
We have presented PIMPC-GNN, a physics-informed multi-consensus framework that addresses class-imbalanced node classification through thermodynamic diffusion, Kuramoto synchronisation, and spectral embedding. By integrating multi-physics principles with adaptive consensus learning, our approach mitigates majority class bias and enables robust minority class recognition. Extensive experiments under extreme imbalance ratios (up to 1:100) demonstrate that PIMPC-GNN consistently outperforms 12 state-of-the-art methods across five benchmarks, with theoretical guarantees and ablation studies validating its robustness and component necessity.

Beyond benchmarks, our approach holds promise for high-impact application domains. The framework shows practical promise for critical applications, including rare disease diagnosis, fraud detection, and infrastructure monitoring, where minority class recognition is essential. Current limitations include static graph assumptions and computational demands for massive-scale graphs. Future work will extend to dynamic graphs, automated physics-weight adaptation, and domain-specific physical law integration, establishing a new paradigm for physics-informed graph learning that bridges theoretical principles with practical imbalanced learning challenges.

\ifCLASSOPTIONcaptionsoff
  \newpage
\fi

\bibliographystyle{IEEEtran}
\bibliography{IEEEabrv,Bibliography}



\begin{IEEEbiography}[{\includegraphics[width=1in,height=1.25in,clip,keepaspectratio]{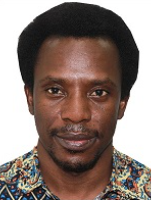}}]{Abdul Joseph Fofanah}
(Member, IEEE) earned an associate degree in mathematics from Milton Margai Technical University in 2008, a B.Sc. (Hons.) degree and M.Sc. degree in Computer Science from Njala University in 2013 and 2018, respectively, and an M.Eng. degree in Software Engineering from Nankai University in 2020. Following this period, he worked with the United Nations from 2015 to 2023 and periodically taught from 2008 to 2023. He is currently pursuing a Ph.D. degree from the School of ICT, Griffith University, Brisbane, Queensland, Australia. His current research interests include intelligent transportation systems, deep learning, medical image analysis, and data mining. 
\end{IEEEbiography}

\begin{IEEEbiography}[{\includegraphics[width=1in,height=1.25in,clip,keepaspectratio]{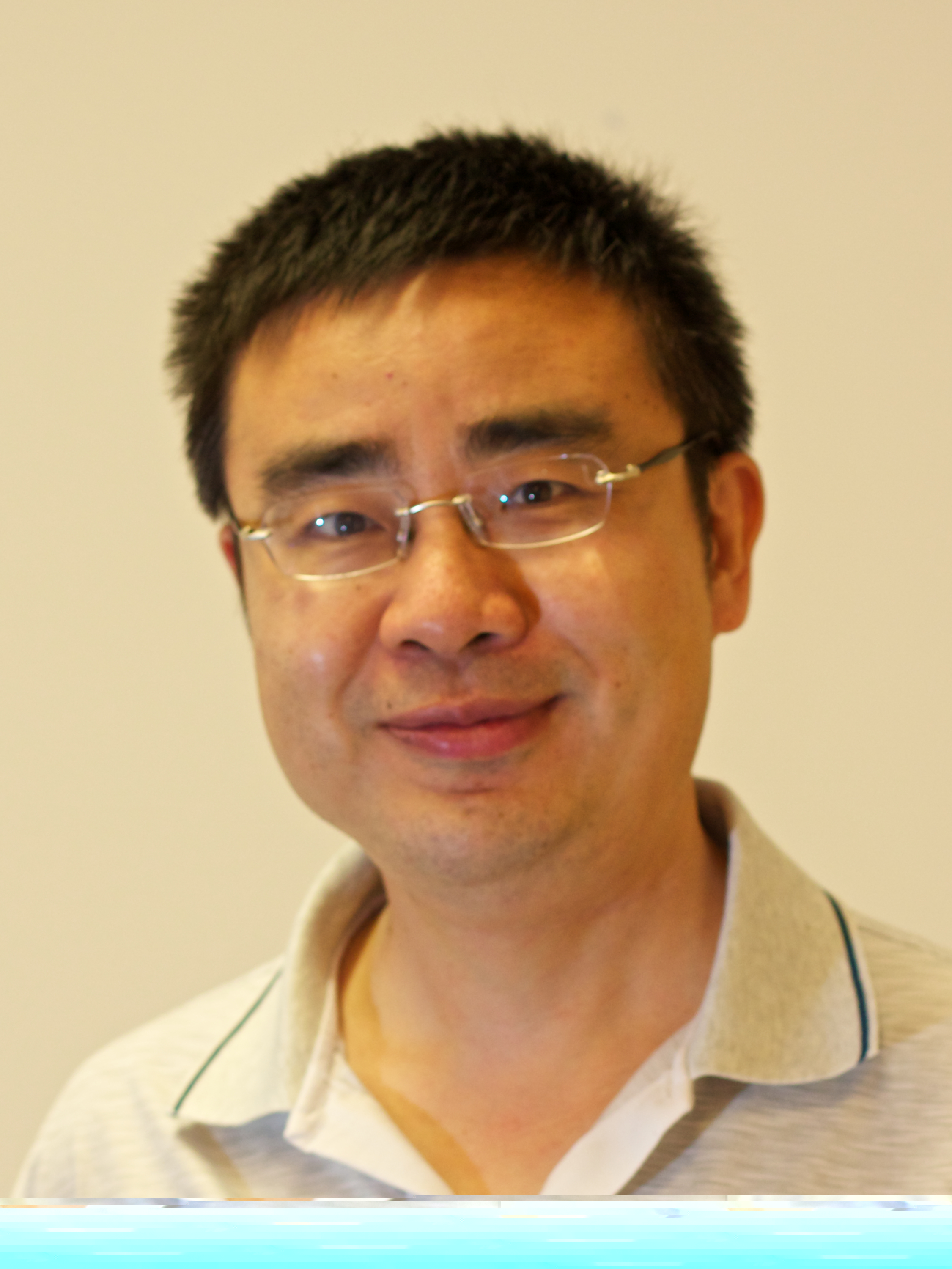}}]{Lian Wen (Larry)}
(Member, IEEE) is currently a Lecturer at the School of ICT at Griffith University. He earned a Bachelor’s degree in Mathematics from Peking University in 1987, followed by a Master’s degree in Electronic Engineering from the Chinese Academy of Space Technology in 1991. Subsequently, he worked as a Software Engineer and Project Manager across various IT companies before completing his Ph.D in Software Engineering at Griffith University in 2007. Larry’s research interests span four key areas: Software Engineering: Focused on Behaviour Engineering, Requirements Engineering, and Software Processes, Complex Systems and Scale-Free Networks, Logic Programming: With a particular emphasis on Answer Set Programming, Generative AI and Machine Learning.
\end{IEEEbiography}

\begin{IEEEbiography}[{\includegraphics[width=1in,height=1.25in,clip,keepaspectratio]{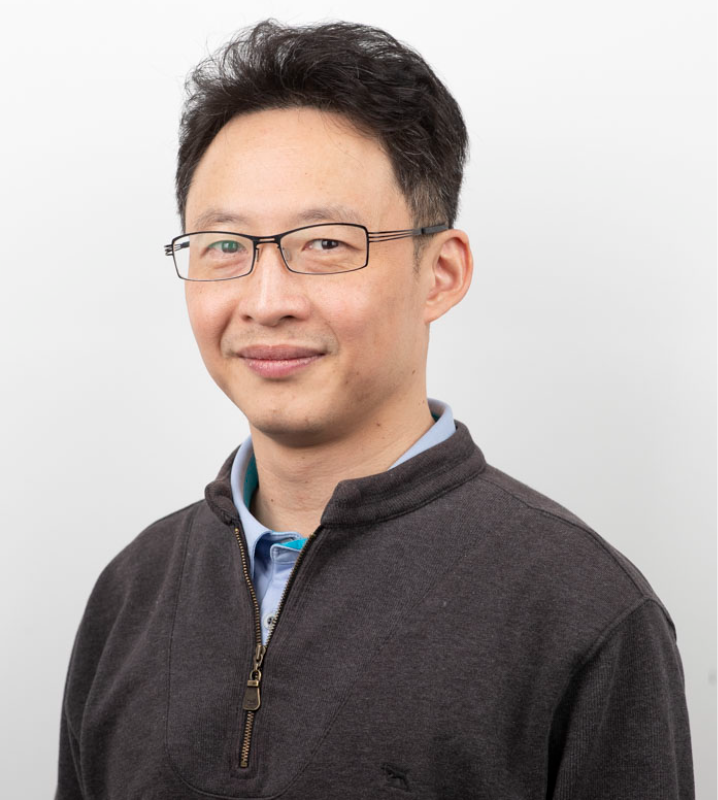}}]{David Chen}
 (Member, IEEE) obtained his Bachelor with first class Honours in 1995 and PhD in 2002 in Information Technology from Griffith University. He worked in the IT industry as a Technology Research Officer and a Software Engineer before returning to academia. He is currently a senior lecturer and serving as the Program Director for Bachelor of Information Technology in the School of Information and Communication Technology, Griffith University, Australia. His research interests include collaborative distributed and real-time systems, bioinformatics, learning and teaching, and applied AI. 
\end{IEEEbiography}


\vfill


\end{document}